%% file: neurips_2025.tex
\documentclass{article}

\usepackage[preprint]{neurips_2025}

\input{variables.tex}

\usepackage[utf8]{inputenc} 
\usepackage[T1]{fontenc}    
\usepackage{hyperref}       
\usepackage{url}            
\usepackage{booktabs}       
\usepackage{amsfonts}       
\usepackage{nicefrac}       
\usepackage{microtype}      
\usepackage{xcolor}         
\usepackage[normalem]{ulem} 
\usepackage{float}
\newif\ifshowchanges
\showchangestrue  

\ifshowchanges
  
  \newcommand{\removed}[1]{{\color{red}\sout{#1}}}
  
\else
  
  \newcommand{\removed}[1]{}
  
\fi

\usepackage{tcolorbox}
\tcbuselibrary{skins,breakable}
\usepackage{graphicx}
\usepackage{tikz}
\usepackage{amsmath, amssymb, amsfonts}
\usepackage{amsthm}
\usepackage{algorithm}
\usepackage{algpseudocode}
\usepackage{float}
\usepackage[algo2e,linesnumbered]{algorithm2e}
\usepackage{threeparttable}

\theoremstyle{plain}
\newtheorem{theorem}{Theorem}[section]

\usepackage[font=small,labelfont=bf]{caption}
\usepackage[font=small]{subcaption}
\usepackage{multirow}
\usepackage{siunitx}
\usepackage{soul} 

\newif\ifinternalcomments
\internalcommentsfalse 

\ifinternalcomments
  \definecolor{commentblue}{rgb}{0,0,0.7}
  \newcommand{\rebecca}[1]{\sethlcolor{commentblue}\hl{[Rebecca: #1]}}
  \definecolor{commentgreen}{rgb}{0,0.5,0}
  \newcommand{\asmit}[1]{\sethlcolor{commentgreen}\hl{[Asmit: #1]}}
  \definecolor{commentorange}{rgb}{1,0.5,0}
  \newcommand{\shikhar}[1]{\sethlcolor{commentorange}\hl{[Shikhar: #1]}}
  \definecolor{commentcyan}{rgb}{0,0.7,0.7}
  \newcommand{\amey}[1]{\sethlcolor{commentcyan}\hl{[Amey: #1]}}
  \definecolor{mydarkblue2}{HTML}{011c73}
  \newcommand{\satwik}[1]{\textcolor{mydarkblue2}{\textit{Satwik:} #1}}
  \newcommand{\khalid}[1]{\sethlcolor{commentorange}\hl{[Khalid: #1]}}
\else
  \newcommand{\rebecca}[1]{}
  \newcommand{\asmit}[1]{}
  \newcommand{\shikhar}[1]{}
  \newcommand{\amey}[1]{}
  \newcommand{\satwik}[1]{}
  \newcommand{\khalid}[1]{}
\fi

\title{{Linear Predictability of Attention Heads in Large Language Models}}

\author{\stepcounter{footnote}
\begin{tabular}{c}
  Khalid Shaikh$^{1}$\thanks{Corresponding author: \href{mailto:kshaikh6@gatech.edu}{kshaikh6@gatech.edu}} \quad
  Asmit Kumar Singh$^{1}$\\[0.4em]
  Rebecca Christopher Dsouza$^{2}$ \quad
  Shikhar Shiromani$^{1}$\\[0.8em]
  {\normalfont
    $^{1}$Georgia Institute of Technology \quad
    $^{2}$Santa Clara University}
\end{tabular}}

\begin{document}

\maketitle

\input{Abstract}
\input{sections/Intro}

\input{sections/linearity}

\input{sections/kv_cache}

\input{sections/relatedwork}
\input{sections/conclusion}

\bibliographystyle{plainnat}
\bibliography{sections/all}

\newpage
\appendix

\input{appendix/add_kvcache}

\input{appendix/add_lin_pred}
\input{appendix/lower_bound}

\input{appendix/resources}

\end{document}

%% file: variables.tex
\usepackage{commath}
\usepackage{amssymb}
\usepackage{amsmath,amsfonts,bm}

\DeclareMathOperator*{\E}{\mathbb{E}}

\DeclareMathOperator{\tr}{\mathsf{tr}}

\def\cN{{\mathcal{N}}}

\newcommand{\br}{\mathbb{R}}
\newcommand{\R}{\mathbb{R}}

\newcommand{\myx}{$\times$}

\def\mA{{\bm{A}}}
\def\mB{{\bm{B}}}
\def\mC{{\bm{C}}}

\def\mI{{\bm{I}}}

\def\mK{{\bm{K}}}

\def\mP{{\bm{P}}}
\def\mQ{{\bm{Q}}}
\def\mR{{\bm{R}}}

\def\mV{{\bm{V}}}
\def\mW{{\bm{W}}}
\def\mX{{\bm{X}}}

\def\vzero{{\bm{0}}}

\def\vb{{\bm{b}}}

\def\vx{{\bm{x}}}

%% file: Abstract.tex
\begin{abstract}
    Large language model (LLM) inference is increasingly bottlenecked by the Key–Value (KV) cache, yet the fine-grained structure of attention-head activations remains poorly understood. We show that pretrained Transformers exhibit a pervasive inter-head linear structure: for a given token, the Query, Key, and Value (QKV) vectors of an attention head can often be reconstructed as a linear combination of a small number of peer heads—typically within the same layer. Across Llama‑3.1 8B, Falcon3‑10B, OLMo‑2‑7B, and Qwen3‑32B, just 2–5 reference heads recover many target heads with high fidelity (e.g., mean $R^2\approx 0.76$ for Keys on C4 with five references, and frequently $R^2 > 0.85$ on GSM8K). This predictability is learned rather than architectural: it is largely absent at random initialization, rises rapidly during pretraining as we track through OLMo‑2 checkpoints, and is supported by a theoretical lower bound showing high mean-squared error for linear prediction at initialization. We further connect emergence to increasing intra-layer alignment of Key projection subspaces. Finally, we exploit this redundancy for efficiency by caching only reference-head KV states and reconstructing the remaining heads on the fly via lightweight linear maps, achieving 2× KV-cache reduction with model-dependent accuracy trade-offs (4.5–5.5 pp average drop on Falcon3‑10B/Qwen3‑32B across five benchmarks, and larger drops on Llama‑3.1 8B), and we find that reconstructing Keys is substantially less harmful than reconstructing Values.
    \footnote{Code: \url{https://github.com/shaikhalid/Interpretability-and-Linear-Predictability-of-Attention-Heads}}
\end{abstract}

%% file: sections/Intro.tex
\section{Introduction}
\label{sec:introduction}

A central question in understanding Large Language Models (LLMs) is what structural properties emerge in their attention mechanisms during pretraining. While prior interpretability work has studied \emph{which} tokens heads attend to and \emph{which} heads can be pruned, the fine-grained relationships among the internal Key, Value, and Query \emph{activation vectors} of different heads remain poorly understood. In this paper, we show that these activations are not independent: in pretrained LLMs, individual heads are strongly \emph{linearly predictable} from their peers, a property absent at initialization and arising as a consequence of pretraining. Beyond its theoretical interest, this structure has a direct practical implication: it enables compression of the KV cache---the primary memory bottleneck during autoregressive generation \citep{kwon2023efficientmemorymanagementlarge}---without any architectural changes.

Previous interpretability work has shed light on attention, revealing functional specializations and significant redundancy between attention heads at a semantic level \citep{clark2019bertattention, michel2019sixteen}. These insights have led to effective techniques for pruning or sharing attention components to alleviate KV cache pressure \citep{agarwal2024chai, ainslie2023gqa, shim2021layerwisepruningtransformerattention}. However, the finer-grained relationships among the \emph{internal K, V, and Q vector activations} of different heads have remained largely unexplored; concurrent work (e.g., \textsc{AQUA-KV}; \citealp{shutova2025cache}) has begun to probe this space in the context of quantization, but the broader structural phenomenon and its implications for head-elimination compression remain uncharacterized. In this paper, we ask: \emph{do pretrained LLMs develop systematic linear dependencies among the activation vectors of different attention heads, and if so, when and how does this structure emerge during training?}

\input{plots/banner_fig}

We give an overview of our contributions below.

\textbf{Identification and Characterization of Linear Predictability (\S\ref{sec:linear-relations}).} We show that individual attention head K, V, and Q activations in pretrained LLMs are \emph{linearly predictable} from a small set of peer heads. Across LLaMA-3.1 8B, Falcon3-10B, OLMo-2-7B, and Qwen3-32B, Key states reach a mean $R^2$ of 0.76 on C4 with $N\!=\!5$ reference heads, and $R^2 > 0.85$ on GSM8K (Fig.~\ref{fig:linear_keystate_suite}). Over 75\% of pairwise head relationships exceed $R^2\!=\!0.5$, revealing a pervasive level of learned intra-layer structure not captured by prior attention-map redundancy analyses.


\textbf{Emergence and Mechanistic Analysis (\S\ref{sec:analysis}).} Linear predictability is absent at initialization (avg $R^2 < 0.05$; fewer than 5\% of head pairs exceed $R^2\!=\!0.10$) and develops during pretraining. Using OLMo-2 intermediate checkpoints (1B to $\sim$4T tokens), we trace a steady rise in intra-layer predictability while inter-layer links peak early and then decline. We connect this trajectory to growing \emph{subspace overlap}: the overlap dimension of Key projection matrices within each layer increases monotonically during training, with Keys expanding fastest, Queries following, and Values remaining nearly orthogonal. A theoretical lower bound (Theorem~\ref{thm:main-lowerbound}) confirms that random initialization produces high MSE with overwhelming probability, establishing that this structure is learned during pretraining, not architectural.

\textbf{KV Cache Compression via Head Elimination (\S\ref{sec:kvcache}).} As a practical application, we compress the KV cache by storing only reference heads and reconstructing target heads on-the-fly with learned linear maps, achieving $2\times$ memory reduction. Accuracy trade-offs are model-dependent: Qwen3-32B and Falcon3-10B incur average drops of 5.5\,pp and 4.5\,pp respectively across five benchmarks, while LLaMA-3.1 8B shows a larger 9.9\,pp drop (Table~\ref{tab:kv-pred-results-main}). Predicting \emph{Keys} is substantially safer than predicting \emph{Values} (\S\ref{subsec:key-vs-value}), consistent with the asymmetric subspace overlap growth in \S\ref{sec:analysis}.


Relation to concurrent work. \cite{shutova2025cache} improve KV-cache quantization by learning compact layer-wise linear predictors—predicting each layer’s keys from the previous layer’s reconstructed keys, and values from the previous layer’s reconstructed values together with current-layer keys—and then storing only quantized residuals. In contrast, we characterize head-level predictability (predominantly within-layer) and exploit it to eliminate KV storage for a subset of heads entirely: we cache a small set of full-precision reference heads and reconstruct the remaining heads on-the-fly using the linear maps identified in §2 and applied in §4.

%% file: plots/banner_fig.tex
\begin{figure}[t]
  \centering
  \begin{minipage}[t]{0.29\linewidth}
    \centering
    \includegraphics[width=\linewidth]{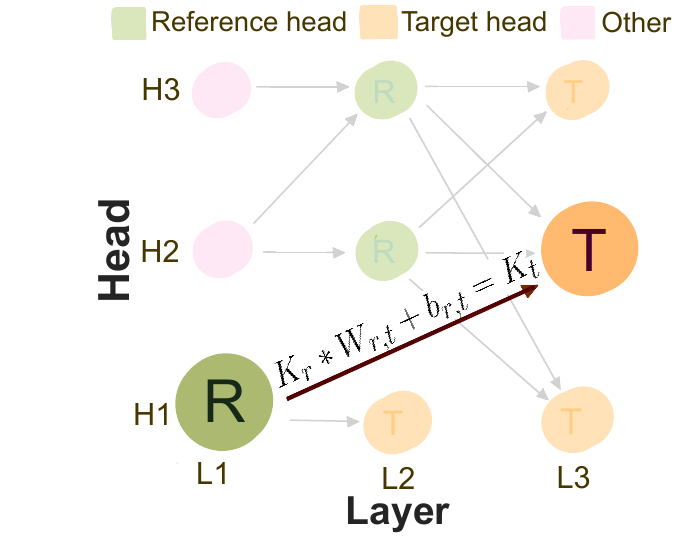}
  \end{minipage}
  \hfill
  \begin{minipage}[t]{0.33\linewidth}
    \centering
    \includegraphics[width=\linewidth]{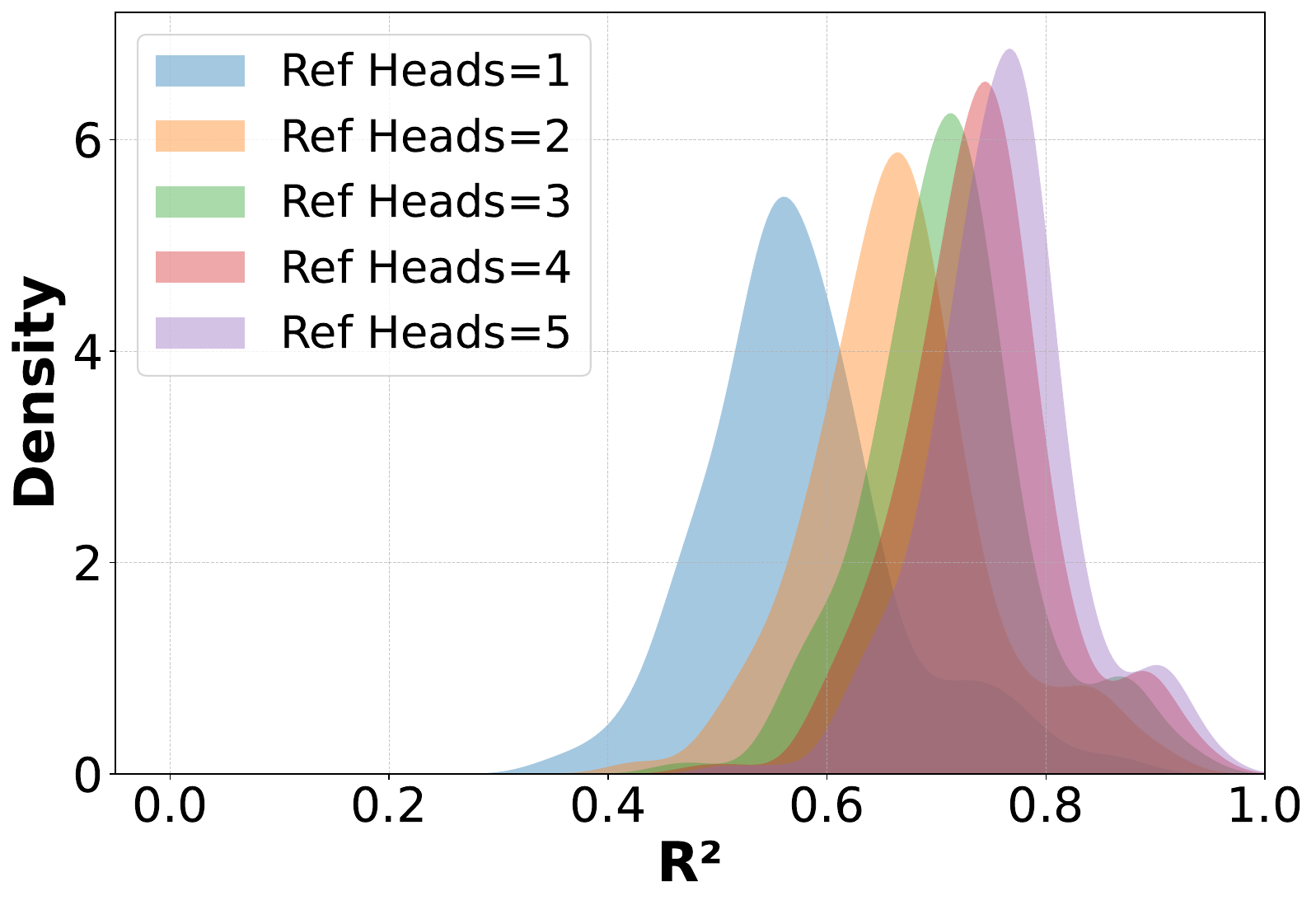}
  \end{minipage}
  \hfill
  \begin{minipage}[t]{0.33\linewidth}
    \centering
    \includegraphics[width=\linewidth]{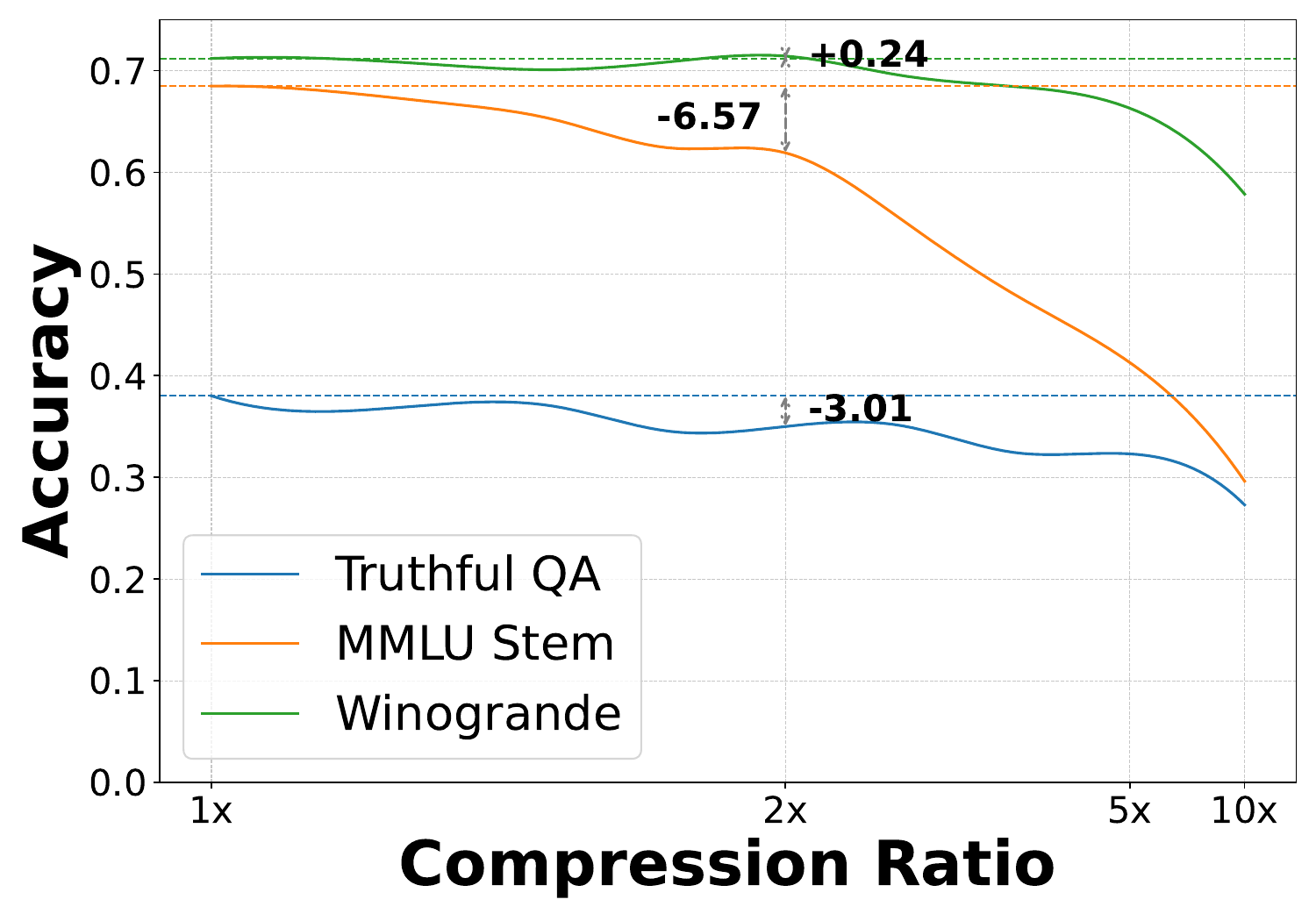}
  \end{minipage}

\caption{\textbf{Linear predictability among attention heads enables significant KV-cache compression.}
(Left) Schematic of approximating a target head's (orange) Key activations using a learned linear projection from reference heads (green).
(Middle) In a pretrained Llama-3.1 8B model, most heads' Key activations are highly predictable (high $R^2$) from a few peers on the C4 dataset \citep{raffel2023exploringlimitstransferlearning} , indicating a shared low-rank subspace.
(Right) This predictability allows for substantial KV-cache compressionwith minimal performance degradation on benchmarks like TruthfulQA, MMLU-STEM, and Winogrande. The dashed line marks 2\myx compression. \vspace{-1em}}

  \label{fig:linear_keystate_suite}
\end{figure}

%% file: sections/linearity.tex
\section{Linear Predictability of Attention Heads}

\label{sec:linear-relations}

A core finding of this work is that pretrained Transformers develop striking linear dependencies among their attention head activations. In this section, we establish and quantify this phenomenon: the Key, Value, and Query activations of individual heads are linearly predictable from a small set of peer heads across diverse model families and tasks. We demonstrate that the activations of a target head can often be accurately reconstructed using a linear combination of activations from a small number of peer heads, a phenomenon that underpins the practical applications discussed later.

\subsection{Background: Multi-Head Attention}
\label{subsec:background_mha}

Our investigation centers on the attention blocks within Transformer-based LLMs. Each layer in a standard Transformer architecture contains a multi-head self-attention (MHA) mechanism followed by a feed-forward network (MLP). The MHA mechanism allows the model to weigh the importance of different tokens when computing the representation for each token.

Specifically, for an input sequence representation $\mX \in \R^{N \times m}$ (where $N$ is sequence length and $m$ is embedding dimension), each of the $H$ attention heads in a layer first projects $\mX$ into Query ($\mQ_h$), Key ($\mK_h$), and Value ($\mV_h$) representations using learned weight matrices $\mW_{Q,h}, \mW_{K,h}, \mW_{V,h} \in \R^{m \times d_h}$ respectively, where $d_h$ is the head dimension. For example, the Key vectors for head $h$ are computed as $\mK_h = \mX \mW_{K,h}$. It is crucial to note that LLMs contain multiple such heads within each layer, and multiple such layers stacked sequentially, forming a deep hierarchy of interacting representations. Our work focuses on the relationships between these individual $\mK_h, \mV_h,$ (and $\mQ_h$) activation vectors across different heads $h$ and different layers.

\input{plots/sec2_linearity_bars_fig} 

\subsection{Empirical Evidence of Linear Predictability}
\label{subsec:empirical_evidence}

To investigate the extent and nature of linear predictability among attention head activations, we conducted experiments using several pretrained LLMs: Llama‑3.1 8B, Falcon3-10B, and OLMo‑2‑7B. For our primary analysis, we collected Key, Value, and Query activations for every head in every layer by processing 150 randomly sampled sequences (average 200 tokens each) from the C4 dataset. This provided approximately 30,000 token activation samples per model for fitting our regression models. Additional datasets like GSM8k were used for further validation, as discussed below.

\textbf{Pairwise Head Predictability.}
Our initial probe involved assessing direct linear relationships between pairs of heads. For every ordered pair of heads $(H_t, H_r)$ within a model (target $H_t$, reference $H_r$, with the constraint $\text{layer}(H_t) \ge \text{layer}(H_r)$ to avoid using future information), we fit an Ordinary Least Squares (OLS) model to predict the activations of $H_t$ using those of $H_r$. The coefficient of determination ($R^2$) from these regressions quantifies the strength of this pairwise linear dependency. As visualized in \autoref{fig:linear_keystate_suite} (middle panel, showing results for Key states in Llama-3.1 8B on C4), a considerable number of head pairs exhibit non-trivial $R^2$ values, suggesting that many individual heads can, to some extent, be linearly predicted by at least one other single head. For instance, 75\% of pairs have $R^2$ > 0.5, and the median $R^2$ is 0.57 for Key states.

\textbf{Enhanced Predictability with Multiple Peer Heads.}
While pairwise analysis revealed individual predictive links, we further explored the potential of using multiple peer heads to enhance the predictability of a target head's activations. For each target head $H_t$, we selected the top $N$ reference heads that yielded the highest $R^2$ scores in the pairwise analysis (these could be from the same or preceding layers). We then assessed their combined predictive strength by fitting a multivariate OLS model using their concatenated state vectors.

The results, depicted in \autoref{fig:linear_keystate_suite} (middle panel, for Llama-3.1 8B Key states on C4), reveal a remarkably strong linear predictability with a small number of peer heads. While adding more reference heads generally increases the average $R^2$, the most significant gains are often achieved with just a few (e.g., 2-3) well-chosen peers. For instance, using $N=3$ reference heads, a substantial fraction of target heads can already be predicted with high fidelity (e.g., a mean $R^2$ of 0.71 is achieved for Key states). Increasing to $N=5$ reference heads pushes this predictability even further, with over 75\% of heads becoming predictable with an $R^2$ > 0.72, and the overall mean $R^2$ reaching a high of 0.76 for Key states.

\textit{\textbf{Takeaway:} Individual head activations are systematically correlated with, and can be accurately approximated by simple linear transformations of a small set of peer heads.}

\subsection{Discussion: Implications of Linear Predictability}
\label{subsec:discussion_linearity}
Two questions naturally follow from the above observations. \textit{(1) Why does the phenomenon matter?} Beyond revealing structural redundancy among attention heads, it shows we can drop many key-value vectors at decoding time and reconstruct them with linear projections, cutting memory usage (see Section \ref{sec:kvcache}). \textit{(2) Why is this linearity surprising?} Predicting heads across layers is unexpected because the vectors pass through nonlinear transformations. Even within a layer, one head can linearly predict another only if their projection subspaces overlap. Concretely, let $\mA,\mB\in\R^{m\times k}$ be projection matrices. For every input $\vx\in\R^{m}$ we seek $\mC\in\R^{k\times k}$ satisfying $\vx^{\top}\mA\mC=\vx^{\top}\mB$, which demands $\mA\mC=\mB$. If the column spaces of $\mA$ and $\mB$ are disjoint, exact prediction is infeasible to achieve; low error can still be attained when the subspaces align or have small principal angles. 
As the next section demonstrates, this linear structure is absent at random initialization and emerges progressively during pretraining---a finding that transforms what could be dismissed as an engineering trick into a fundamental structural property of learned Transformers.

\input{plots/subspace_overlap_and__inter-intra_linearity}
\section{Insights on Linear Predictability of Heads}
\label{sec:analysis}

The observation of strong linear relationships among attention head activations (Section~\ref{sec:linear-relations}) prompts further investigation into their origin and development. This section examines how and when this structure arises, addressing two key questions: (1) Is linear predictability an intrinsic property of Transformers, or is it learned? (2) How does this predictability evolve throughout the pretraining process? Our findings indicate that linear predictability is negligible at initialization and is indeed a learned byproduct of pretraining. We also provide a theoretical argument underscoring why such predictability is unlikely in randomly initialized networks. To our knowledge, no prior work---including concurrent methods that exploit inter-head or inter-layer linearity for cache compression \citep{shutova2025cache}---has studied when or why these linear dependencies arise; the analysis in this section addresses this gap.

\textbf{Evolution of Linearity: An Emergent Property of Pretraining.}
We first contrast the linear predictability observed in fully pretrained language models with their randomly initialized (untrained) counterparts. By analyzing Llama-3.1 8B on the C4 dataset, we find that prior to any pretraining, the average $R^2$ for predicting one head's Key activations from another single head is exceptionally low (with average $R^2 < 0.05$). This starkly contrasts with the significantly higher $R^2$ values we observe in the fully pretrained version of the same model. This comparison strongly suggests that the observed linear structure and inter-head redundancy are not inherent to the architecture but rather emerge and are refined during the pretraining phase.

As \autoref{fig:subspace_overlap_r2_density} shows, pre-training transforms the
head-to-head relationships in LLaMA–3.1 8B from almost independent
to strongly linearly predictable.  
In the randomly initialized network (orange), fewer than \(5\%\) of head pairs
exceed \(R^{2}=0.10\); nearly the entire mass is concentrated at the origin.
After optimization (blue), the distribution shifts dramatically rightward:
the median climbs to \(R^{2}\!\approx\!0.50\), roughly half of all pairs now
surpass \(0.50\), and the top decile approaches \(0.80\).  
This broad peak around \(0.55\) confirms that the training process induces a shared low-rank sub-space in which one head’s
activations can be reconstructed from a small set of others.

Furthermore, analyzing the nature of these connections, we observe distinct trends for intra-layer versus inter-layer predictability as training progresses. ~\autoref{fig:intra-inter-layer-combined} shows the fraction of the strong connections that are intra-layer versus inter-layer at each OLMo-2 checkpoint. Initially, inter-layer is relatively higher, but as pretraining advances, intra-layer predictability steadily strengthens and becomes dominant, while strong inter-layer predictive links become less prevalent.


\begin{center}
  \begin{minipage}[t]{0.49\textwidth}
    \centering
    \includegraphics[width=0.6\linewidth]{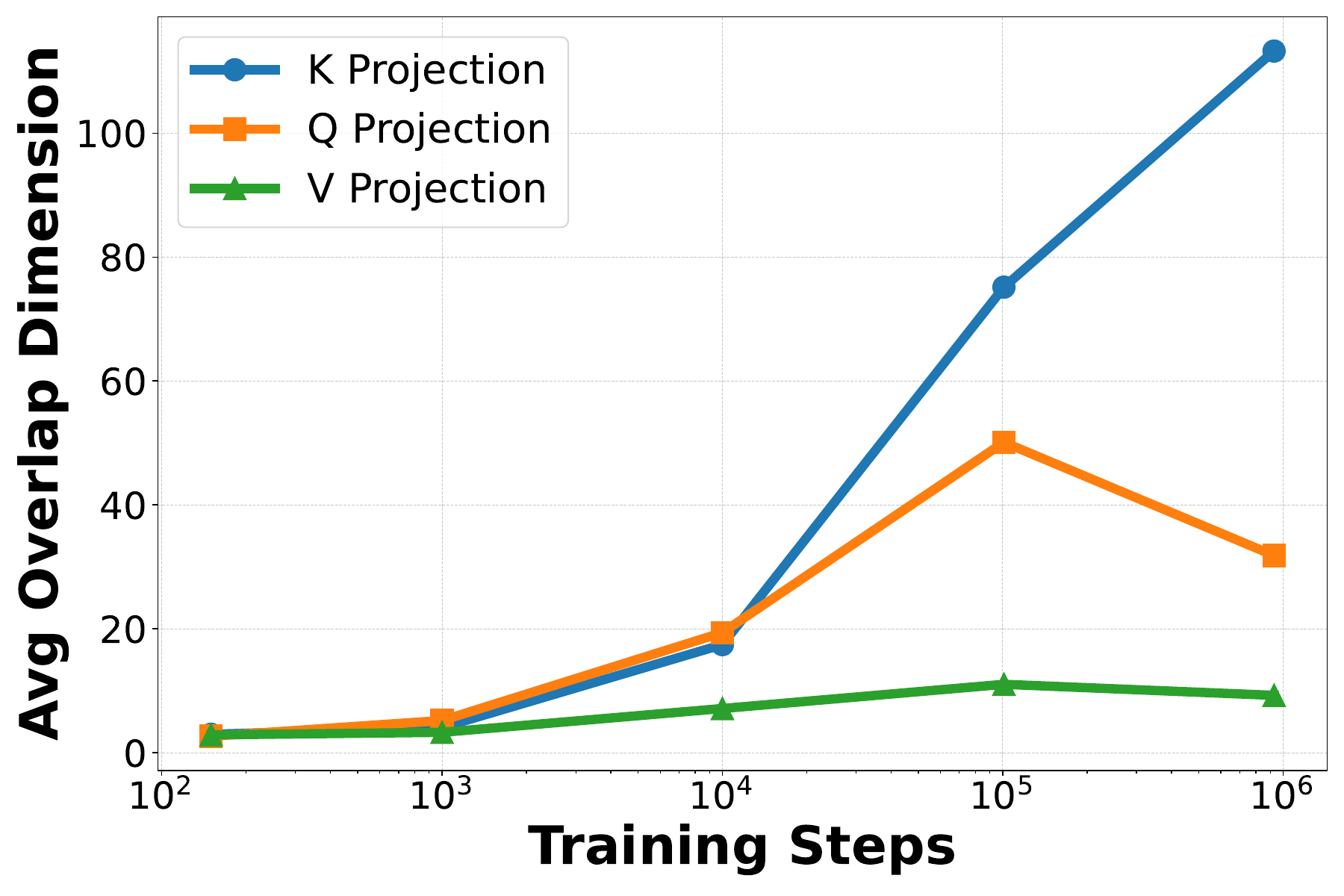}
  \end{minipage}\hfill
  \begin{minipage}[t]{0.49\textwidth}
    \centering
    \includegraphics[width=0.6\linewidth]{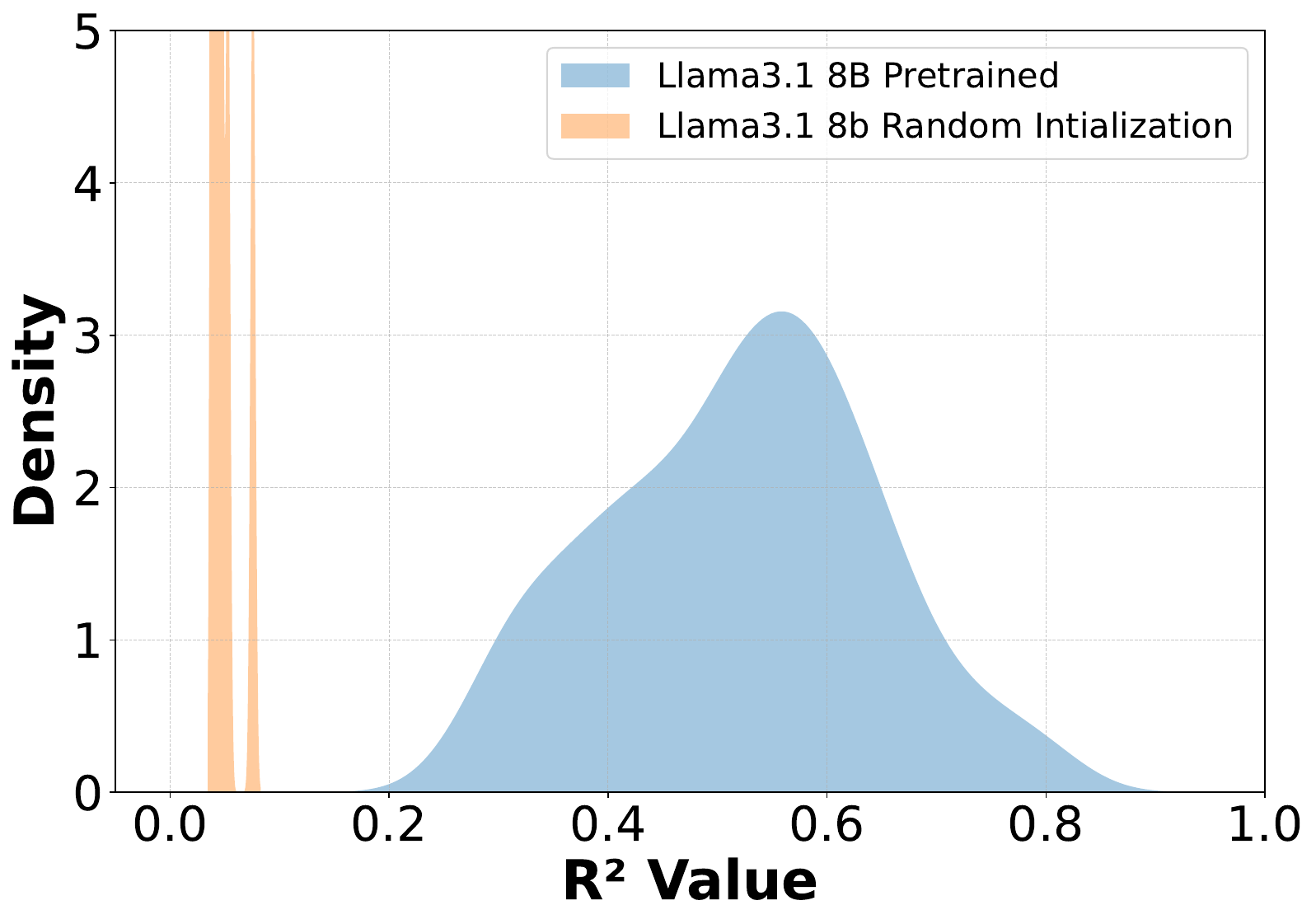}
  \end{minipage}
  \captionsetup{hypcap=false}
  \captionof{figure}{\textbf{Left: Pretraining aligns head subspaces.} Overlap dimension averaged across layers at different points during OLMo-2 training, indicating increasing intra-layer alignment. \textbf{Right: Training induces linear predictability.} Kernel-density estimate of $R^{2}$ for LLaMA-3.1 8B shows a broad moderate–high peak after training versus near-zero values at random initialization.}
  \label{fig:subspace_overlap_r2_density}
\end{center}

To gain finer-grained insight into this emergence, we use the publicly released intermediate checkpoints of the OLMo-2 model. This allows us to track the evolution of linear predictability from random initialization through various stages of pretraining (1B, 10B, 100B, and 1T tokens, up to the final model trained on $\approx$ 4T tokens). Concurrently, to analyze the structural origin of this linearity, we examine the overlap of the attention activation subspaces within layers for OLMo-2. Let $\mW_K^{(\ell)} = [\mW_{K,1}^{(\ell)}, \ldots, \mW_{K,H}^{(\ell)}] \in \R^{m \times H d_h}$ be the concatenation of all key projection matrices in a specific layer $\ell$. We measure the shared subspace using the \textit{overlap dimension}:
\[
\operatorname{OD}_H(\mW_K^{(\ell)}) = \sum_{i=1}^{H}\operatorname{rank}(\mW_{K,i}^{(\ell)}) - \operatorname{rank}(\mW_K^{(\ell)}).
\]
A higher $\operatorname{OD}_H$ indicates greater linear dependence among the column spaces of the individual head projection matrices within that layer. We measure the intra-layer subspace overlap for key, value, and query projection vectors at each checkpoint. \autoref{fig:subspace_overlap_r2_density} illustrates this evolution. At the very first checkpoint (100 steps) all three projections are nearly orthogonal—each layer shares less than 5 basis directions among its heads. As training unfolds, the overlap dimension grows monotonically—but not uniformly—across the three projections. Key subspaces expand the fastest, showing a steep, almost exponential rise that dominates late-stage checkpoints. Query projections follow the same trajectory, albeit with a gentler slope, while Value projections remain largely flat, indicating near-orthogonality throughout. This growing subspace overlap within layers directly mirrors, and likely underpins, the observed rise in intra-layer linear predictability of head activations.

\input{plots/section2_combined_fig} 

\textbf{Theoretical Analysis: Why Predictability is Low at Random Initialization?}
To further understand why significant linear predictability is absent before training, we theoretically analyze the case of randomly initialized projection matrices. We show that, with high probability, the best linear map to predict one randomly projected vector from another (or a set of others) incurs a substantial mean-squared error. Specifically, for a single-layer Transformer where input token embeddings $\vx$ are drawn from $\mathcal{N}(\vzero, \mI_m)$ and Key projection matrices $\mA, \mB \in \R^{m \times d_h}$ have i.i.d. Gaussian entries, the expected squared error of the best linear predictor is lower bounded by $\Omega(d_h m)$. This bound also holds if $\mA$ represents the concatenation of $n$ reference head projections ($m \times n d_h$) and $\mB$ is a single target head projection ($m \times d_h$). This theoretical result (Theorem~\ref{thm:main-lowerbound}) formally establishes that strong linear predictability between arbitrary head activations is not expected in a randomly initialized state, corroborating our empirical findings and highlighting that such structure must be learned.

\begin{theorem}
    \label{thm:main-lowerbound}
    Let $\mA, \mB \in \br^{m \times k}$ be two random matrices where each entry is sampled independently from $\mathcal{N}(0, 1)$ and $k \leq m/2$. Let $\vx \in \br^{m}$ be an input vector drawn from $\mathcal{N}(\vzero, \mI_m)$ independently of $\mA, \mB$. Then, there is an absolute constant $c_1>0$ such that
    \[
    \Pr\Bigl\{ 
    \inf_{\mC\in\R^{k\times k}}\E_{\vx}\bigl\|\vx^{\top}\mA\mC-\vx^{\top}\mB\bigr\|_2^2
     < \frac{1}{2}k(m-k)
    \Bigr\}
     \le 2\exp\bigl(-c_1 m\bigr).
    \]
    Equivalently, with probability at least $1-2e^{-c_1 m}$ we have
    \[
    \inf_{\mC}\Bigl [\E_{\vx \sim \mathcal{N}(\vzero, \mI_m)}\|\vx^{\top}\mA\mC-\vx^{\top}\mB\|_2^2 \Bigr ]
     \; \geq \; \frac{1}{2}mk.
    \]
\end{theorem}

\begin{proof}[Proof Sketch.]
The proof of the theorem is in Appendix~\ref{app:proof_lower}. We provide a sketch here. The problem reduces to deriving a lower bound for the matrix $\|\mA\mC-\mB\|_F^2$. See that, for any fixed $\mC \in \R^{k \times k}$, we have

\[
\E_{\vx \sim \mathcal{N}(\vzero, \mI_m)}\bigl\|\vx^{\top}(\mA\mC-\mB)\bigr\|_2^2
=\tr\bigl((\mA\mC-\mB)^{\top}(\mA\mC-\mB)\E_{\vx}[\vx \vx^{\top}]\bigr) =\|\mA\mC-\mB\|_F^2. 
\]

We use the fact that any Gaussian random matrix such as $\mA, \mB$ has full column rank almost surely. The regression problem $\inf_C\|\mA\mC-\mB\|_F^2$ has a unique minimizer $\mC^{*} = \mA^{\dagger}\mB$ where $\mA^{\dagger}$ is the Moore-Penrose pseudoinverse. The key insight then is to rewrite the expression $\|\mA\mA^{\dagger}\mB-\mB\|_F^2$ as $\|\mR\mB\|_F^2 = \sum_{j=1}^{k}\|\mR\vb_j\|_{2}^2$ where $\mR$ satisfies certain properties so that one can apply the Hanson-Wright inequality \cite[Thm. 6.2.2]{Vershynin_2018} to obtain the final bound. 
\end{proof}

\textbf{Discussion of Theoretical Result.}
The theorem formalizes the intuition that the column spaces of two independently drawn random projection matrices $\mA$ and $\mB$ are unlikely to be well-aligned. Therefore, projecting an input vector $\vx$ through them will typically yield vectors whose relationship is not well captured by a simple linear transformation, leading to high reconstruction error. This contrasts sharply with the low error (high $R^2$) observed in pretrained models, emphasizing that the alignment of these projection subspaces is a key outcome of the learning process. Since one of the key steps of the proof is the application of Hanson-Wright inequality, which holds for any subgaussian vectors, the proof can naturally be extended to other common initialization methods such as Xavier uniform or Xavier normal initialization \citep{glorot2010understanding}.

The monotonic growth of Key subspace overlap during pretraining, coinciding with the rise in linear predictability, suggests that gradient updates across heads sharing the same input $\mX$ may drive projection matrices toward aligned column spaces. While a full theoretical characterization of this process remains open, the empirical trajectory---from near-orthogonal initialization to substantially overlapping trained subspaces---indicates a systematic, not accidental, alignment. Note that while Theorem~\ref{thm:main-lowerbound} formally applies to the first layer under Gaussian inputs, the empirical evidence ($R^2 < 0.05$ across \emph{all} layers at initialization) suggests the qualitative conclusion extends throughout the network.

%% file: plots/sec2_linearity_bars_fig.tex
\begin{figure}[t]
  \centering
  \begin{minipage}[t]{0.3\linewidth}
    \centering
    \includegraphics[width=\linewidth]{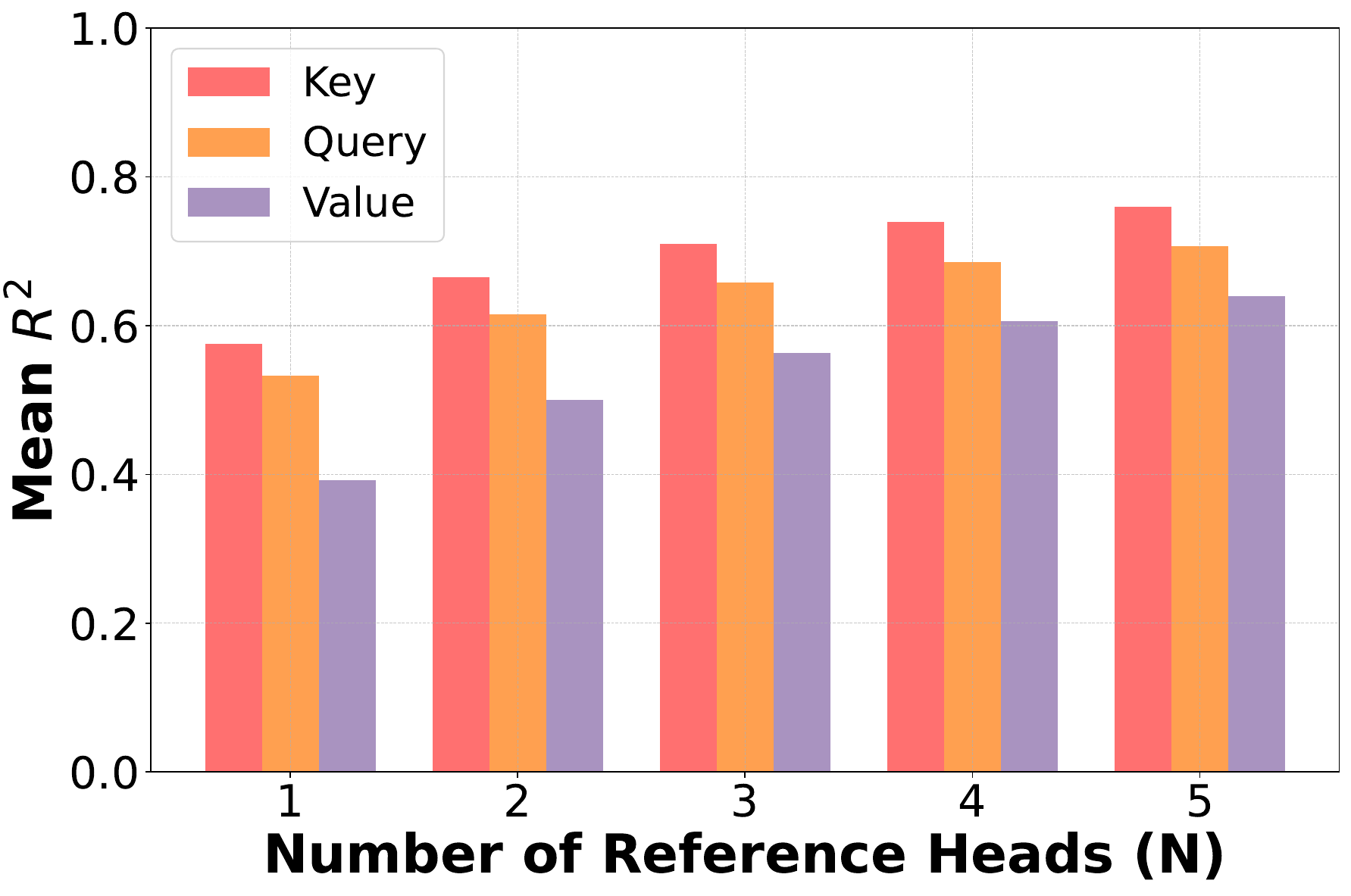}
  \end{minipage}
  \hfill
  \begin{minipage}[t]{0.3\linewidth}
    \centering
    \includegraphics[width=\linewidth]{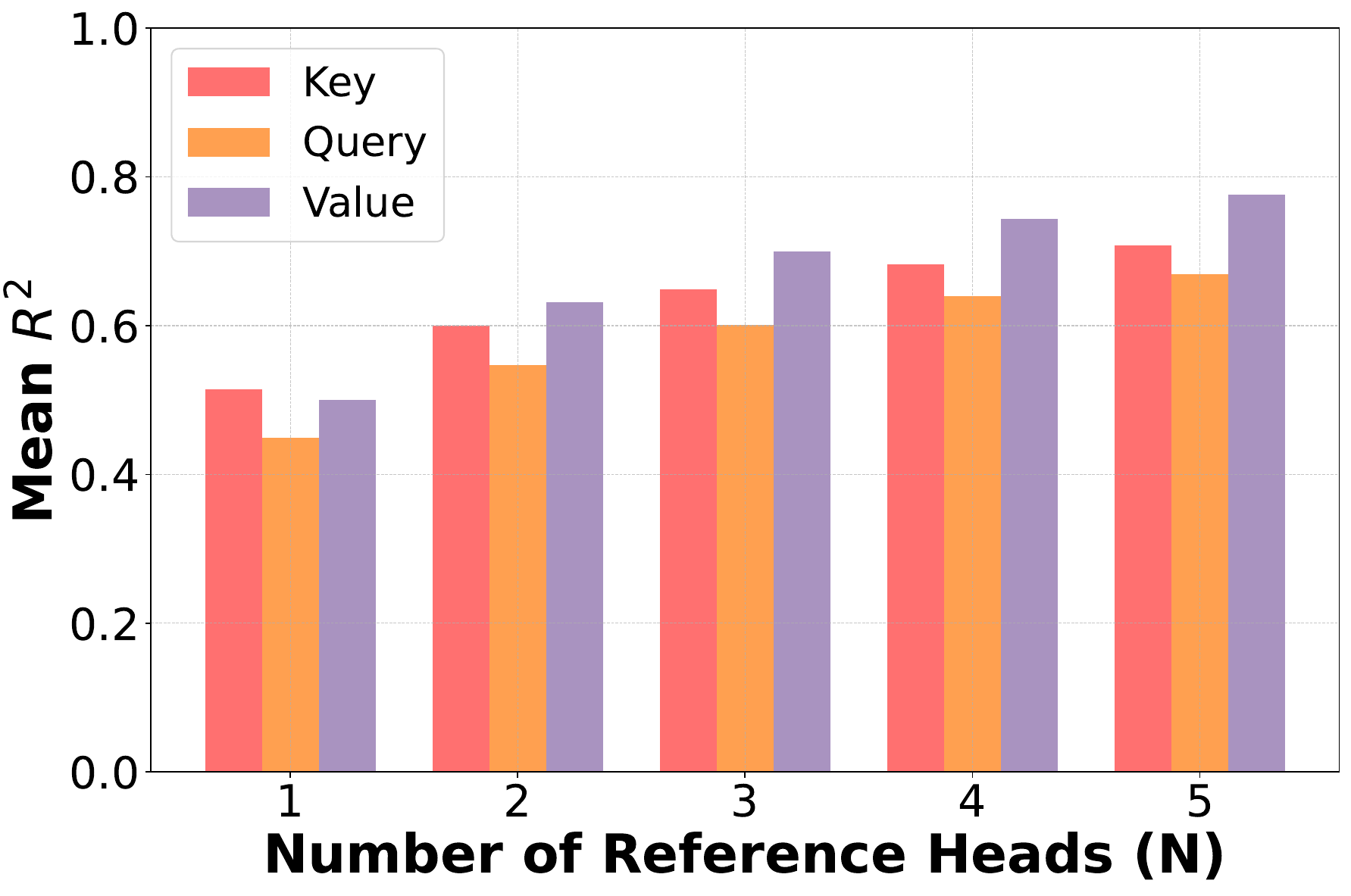}
  \end{minipage}
  \hfill
  \begin{minipage}[t]{0.3\linewidth}
    \centering
    \includegraphics[width=\linewidth]{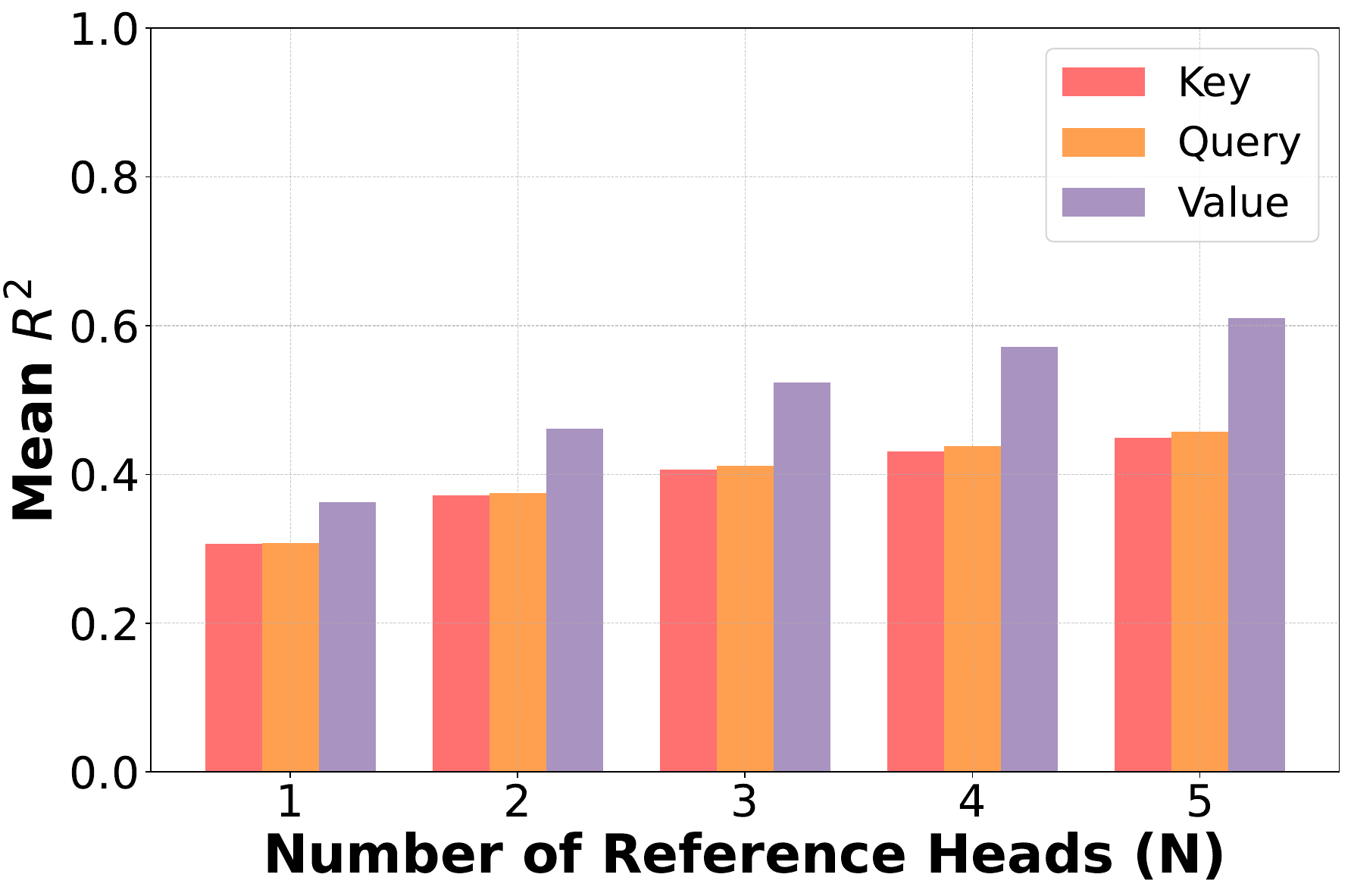}
  \end{minipage}
\caption{\textbf{Linearity is pervasive—keys, queries, and values all become more predictable as the number of reference heads grows.}
Shown is the mean $R^{2}$ when reconstructing {\textbf{key}}, 
{\textbf{query}}, and {\textbf{value}} 
activations on a shared 150-sequence
slice of C4 for three pretrained models:
(Left) LLaMA-3.1 8B (Middle) Falcon-3 10B (Right) OLMo-2 7B.
Across all panels, every stream (K, Q, V) shows a monotonic rise in $R^{2}$ as
additional references are added, confirming that head-level representations in
each model occupy a low-rank sub-space that can be captured with only a handful of heads. \vspace{-1.5em}}
  \label{fig:r2_trends}
\end{figure}

%% file: plots/subspace_overlap_and__inter-intra_linearity.tex
\begin{figure}[t]
  \centering
  \begin{minipage}[t]{0.44\linewidth}
    \centering
    \includegraphics[width=\linewidth]{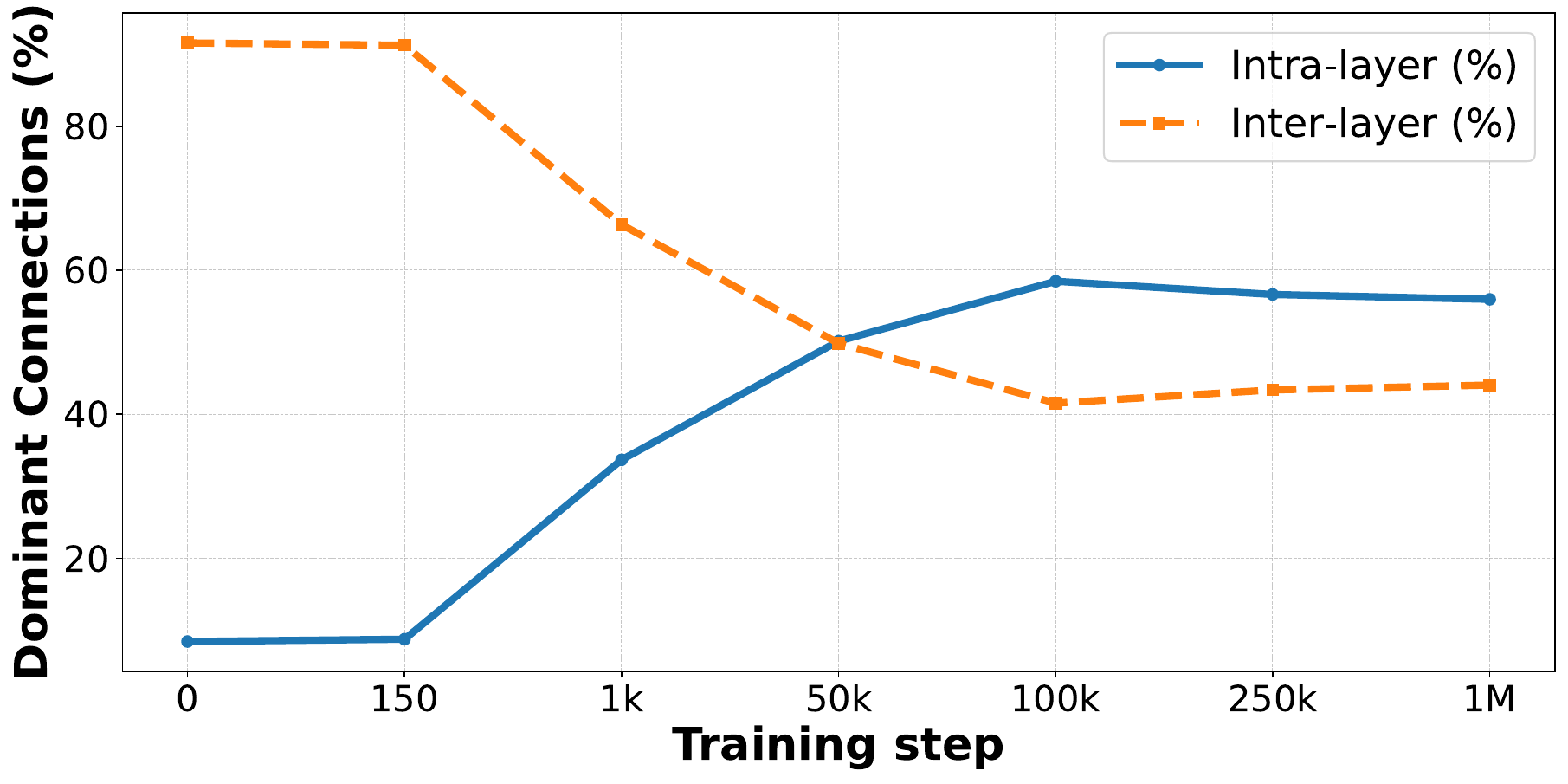}
    \label{fig:intra-inter_layer_connection_percentage1}
  \end{minipage}
  \hfill
  \begin{minipage}[t]{0.44\linewidth}
    \centering
    \includegraphics[width=\linewidth]{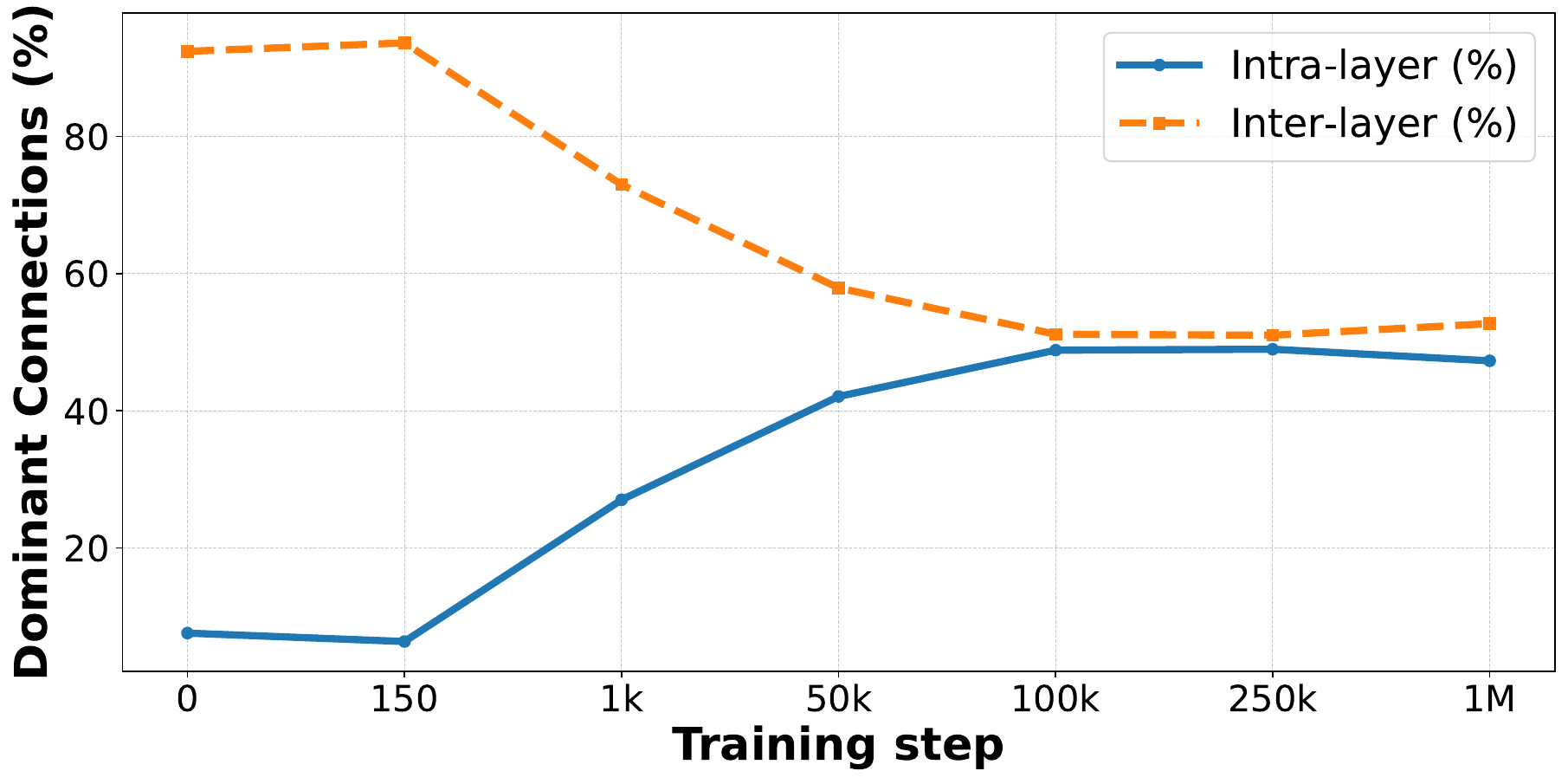}
    \label{fig:intra-inter_layer_connection_percentage2}
  \end{minipage}
  \caption{\textbf{Dominant predictors migrate from cross-layer to within-layer}. For OLMo-2 7B we track, across checkpoints, where a head’s strongest linear predictors reside.
(Left): percentage of heads whose \emph{single} best ($R^{2}$-max) predictor lies in the same layer versus a different layer.
(Right): the same percentages after aggregating each head’s \emph{five} best predictors.
At initialization almost all dominant links are inter-layer; by 50 k–100 k steps intra-layer links overtake and remain prevalent, indicating that pre-training progressively concentrates shared computation inside individual layers. \vspace{-1.5em}}

  \label{fig:intra-inter-layer-combined}
\end{figure}

%% file: plots/section2_combined_fig.tex


%% file: sections/kv_cache.tex
\section{Exploiting Linear Predictability for KV Cache Compression}
\label{sec:kvcache}

As established in Section~\ref{sec:linear-relations}, the internal Key (K) and Value (V) activations of attention heads in pretrained LLMs exhibit significant linear predictability. This finding opens a new avenue for optimizing LLM inference, particularly by addressing the memory bottleneck of the KV cache. During autoregressive decoding, LLMs generate tokens one by one. For each newly generated token, the model attends to all previously processed tokens. To avoid recomputing attention over the entire history at each step, the Key and Value vectors computed for each token in the sequence are stored in a \textit{KV cache} \citep{kwon2023efficientmemorymanagementlarge, agrawal2024tamingthroughputlatencytradeoffllm,zheng2024sglangefficientexecutionstructured,280922} As the sequence length grows, this cache can become very large, consuming a substantial portion of available GPU memory. This memory footprint can be a critical bottleneck, limiting the maximum batch size, context length, or even the feasibility of deploying very large models on consumer devices with limited GPU memory.

The linear predictability we observed suggests a direct strategy for KV cache compression: if the K and V states of a ``target'' head can be accurately reconstructed from a linear combination of K and V states from a few ``reference'' heads, we only need to store the KV states for the reference heads and a small set of learned projection weights.
Our approach differs fundamentally from concurrent predictor+quantization methods such as \textsc{AQUA-KV} \citep{shutova2025cache}: rather than storing quantized residuals for \emph{all} heads after subtracting a layer-to-layer predicted baseline, we \emph{eliminate} storage of target heads entirely, retaining full-precision states only for reference heads and reconstructing the rest on-the-fly via the linear maps learned in \S\ref{sec:linear-relations}. This head-elimination strategy directly leverages the intra-layer linear structure characterized above.
In the rest of this section, we outline the methodology for this strategy.

\input{diagrams/kv_cache_pred_flow_fig} 

\subsection{KV Cache Compression Methodology}
\label{subsec:kv_methodology}

Our KV cache compression strategy involves two main phases. First, a \textbf{calibration phase} is performed once per model to identify candidate heads for prediction and to learn the necessary linear projection weights. Second, during the \textbf{inference phase}, these learned projections are used to reconstruct the KV states of designated target heads on-the-fly, thereby reducing the need to store them explicitly. Figure \ref{fig:kvcachepredflow} illustrates this overall process.

During the \textbf{Calibration Phase}, we first collect a representative set of Key and Value activation states on a small calibration dataset. We perform forward passes on a small, randomly sampled subset (e.g., 1--5\% of a diverse dataset like C4, or a task-specific dataset if learning specialized projectors) and store the K and V activations from all heads at each layer. Based on these activations, we select a subset of heads whose KV states will be explicitly stored (reference heads, $\mathcal{H}_{\text{ref}}$) and another subset whose KV states will be predicted (target heads, $\mathcal{H}_{\text{target}}$). To ensure robust predictions, we first quantify pairwise linear relationships between all heads using the calibration data, as described in \autoref{subsec:empirical_evidence}. Based on these pairwise $R^2$ scores, we construct a directed graph $G=(V,E)$ where nodes $v \in V$ are attention heads and a weighted edge $(h_r \to h_t)$ captures the $R^2$ score between the reference head $h_r$ and target head $h_t$.

We then employ the \texttt{Selection of Target Heads to be Predicted} algorithm (Algorithm\ref{alg:ref_target_selection}) to partition the set of all heads $V$ into reference heads $V\setminus T$ and target heads $T$. The primary goal is to designate a fraction $f$ of heads as targets (e.g., $f=0.5$ for aiming for $2\times$ compression), while ensuring robust predictability for each target. The algorithm achieves this through a multi-step process. First, it performs a binary search for an optimal $r^{2}$ threshold $\tau$. This threshold is used to prune weak predictive edges from the initial graph $G$. During this process, we consider a head as a valid prediction candidate if it has at least $m$ incoming edges from other heads to ensure high quality predictions. The algorithm greedily selects heads for $T$ from the pool of valid candidates. Finally, for each selected target head $t$, if it has more than $m$ valid reference peers in $G^*$ after this process, any additional peers beyond the $m$ strongest ones (those with the highest $r^{2}$ values) are pruned away. This ensures that each target relies on a fixed, small number of its best predictors, keeping the learned linear maps compact. Heads not selected as targets but having outgoing predictive edges in $G^*$ form the set $V\setminus T$. Once the reference and target heads are identified, we learn the multivariate linear projections for each target head similar to the method described in Section\ref{subsec:empirical_evidence}.

\input{algorithms/ref_target_mapping} 

During the \textbf{Inference Phase}, we only store the KV states for the reference heads $\mathcal{H}_{\text{ref}}$ and reconstruct the KV states for the target heads $\mathcal{H}_{\text{target}}$ using the learned linear projections. By only storing KV states for $|\mathcal{H}_{\text{ref}}|$ heads and the small weight matrices $\mW^*$, we significantly reduce the overall KV cache memory footprint.

\subsection{Experiments and Results}
\label{subsec:kv_experiments_results}

Next, we evaluate the efficacy of our KV cache compression method. We focus on three main aspects:
\begin{enumerate}
    \item \textbf{Compression Efficacy}: What is the impact on downstream task accuracy when using our method to predict (e.g., 50\% of) KV states, compared to baseline full-cache inference?
    \item \textbf{Cross-Task Transferability}: Can linear regressors trained on one task generalize effectively to KV prediction on unseen tasks, both within and across domains?
\end{enumerate}

\textbf{Models and Benchmarks.} We validate our KV cache compression approach on Llama3 (8B), Falcon3 (10B), and Qwen3 (32B) model families. Benchmarks include Hellaswag \citep{zellers2019hellaswagmachinereallyfinish}, TruthfulQA \citep{lin2022truthfulqameasuringmodelsmimic}, PIQA \citep{bisk2019piqareasoningphysicalcommonsense}, ARC \citep{clark2018thinksolvedquestionanswering}, MMLU \citep{hendrycks2021measuringmassivemultitasklanguage}, MMLU Pro \citep{wang2024mmluprorobustchallengingmultitask}, Humaneval \citep{chen2021evaluatinglargelanguagemodels}, and MATH \citep{lewkowycz2022solvingquantitativereasoningproblems}. All experiments were conducted on H100-NVL GPUs with 80GB of HBM memory.

\textbf{Compression vs. Accuracy.} \autoref{tab:kv-pred-results-main} reports the primary results, comparing baseline accuracy with our compression method implementing 2\myx KV cache compression (predicting 50\% of heads). Across models, this approach achieves substantial memory reduction with accuracy trade-offs that vary by architecture. Qwen3-32B incurs an average accuracy drop of 5.5\,pp; Falcon3-10B incurs 4.5\,pp; LLaMA-3.1 8B shows a larger 9.9\,pp drop (Table~\ref{tab:kv-pred-results-main}).

\input{tables/KV_predict_main_results}

\textbf{Transferability of Linear Relations Across Tasks} A key question for practicality is whether the learned linear projections are task-specific or if they capture more general inter-head relationships. We investigate if regressors trained on a source task can effectively predict KV states on unseen target tasks.

Our experiments, summarized in \autoref{tab:key_value_prediction_table}, reveal that these linear relationships exhibit considerable transferability, especially \textit{within} the same domain (e.g., math to math). For instance, projections learned from GSM8K generalize well to other mathematical reasoning tasks. When transferring across disparate domains (e.g., math to code generation or psychology), performance can degrade, suggesting that optimal projections might capture some domain-specific nuances. However, cross-domain transfer is not always detrimental. Intriguingly, projections learned from the Psychology domain demonstrate strong generalization to MMLU Law (improving accuracy from a baseline of 22.25 to 24.88) and MMLU Health (accuracy drops by only $\approx 1.5\%$). This suggests that some learned relationships, perhaps those related to abstract or contextual reasoning, can be broadly applicable across human-centered reasoning tasks.

\input{tables/key_and_value_prediction}

\subsection{Disentangling Key and Value Compression}
\label{subsec:key-vs-value}

The subspace overlap analysis in \S\ref{sec:analysis} revealed that Key projection matrices develop far greater intra-layer alignment than Value matrices during pretraining---Key subspaces expand fastest while Value subspaces remain nearly orthogonal throughout training. This asymmetry predicts that Key states should be more amenable to linear reconstruction than Value states. We now test this prediction directly by disentangling the two streams.

Table~\ref{tab:kv-asymmetry-summary} reports average accuracy on the GSM-8k benchmark on Llama3.1 8b under three settings: Key-only prediction (half of Key heads reconstructed, Values stored normally), Value-only prediction, and joint K+V prediction (the setting evaluated in Table~\ref{tab:kv-pred-results-main} and Table~\ref{tab:key_value_prediction_table}). A clear hierarchy emerges.

\begin{table}[ht]
\centering
\caption{Average accuracy on the GSM8k under three compression settings (50\% of heads reconstructed in each case). Keys are substantially more resilient to linear approximation than Values, consistent with the asymmetric subspace overlap growth observed during pretraining (\S\ref{sec:analysis}).}
\label{tab:kv-asymmetry-summary}
\begin{tabular}{lcc}
\toprule
Compression setting & Avg.\ accuracy & Drop vs.\ baseline \\
\midrule
Baseline (no compression)            & 36.77\%  & ---           \\
Key-only (50\% of keys predicted)    & 34.80\%  & $-$1.97\,pp   \\
Value-only (50\% of values predicted)& 32.47\%  & $-$4.30\,pp   \\
Joint K+V (50\% of each predicted)   & 30.04\%  & $-$6.73\,pp   \\
\bottomrule
\end{tabular}
\end{table}

\textbf{Key compression is the least disruptive}: reconstructing half the Key heads drops average math accuracy by only 1.97\,pp. Value compression is more harmful at 4.30\,pp, and joint K+V compression is worst at 6.73\,pp. Furthermore, even a projector calibrated on a semantically distant domain (Psychology) can compress Keys while retaining roughly 92\% of baseline accuracy on math benchmarks; the same cross-domain projector applied to Values degrades performance far more sharply. Full per-task results for Key-only and Value-only compression are in Appendix~\ref{app:kv_cache}.

This asymmetry has a natural interpretation: because Key projection subspaces become highly aligned during pretraining, the linear maps between Key heads are well-conditioned and generalise across calibration domains. Value subspaces, remaining largely orthogonal, retain more task-specific structure that a single linear map cannot easily approximate.

%% file: diagrams/kv_cache_pred_flow_fig.tex
\begin{figure}[h]            
  \centering
  \includegraphics[width=0.9\linewidth]{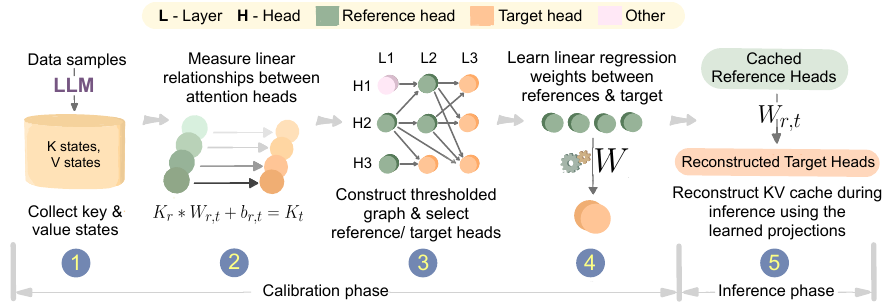}
  \caption{\textbf{Pipeline for KV-cache compression via head-level redundancy.}
(1) We log key/value activations on a calibration set.  
(2) Pair-wise linear probes quantify how well one head predicts another.  
(3) A thresholded graph selects a minimal reference set that covers all targets.  
(4) Compact weight matrices \(W_{r\!\to t}\) are trained for those links.  
(5) During inference, we cache only the reference heads and reconstruct the others on-the-fly, shrinking memory with minimal extra compute. }
  \label{fig:kvcachepredflow}
\end{figure}

%% file: algorithms/ref_target_mapping.tex
\begin{tcolorbox}[
    float,                    
  floatplacement=h,
  colback=gray!10,
  colframe=black!80,
  title={Selection of Target heads to be Predicted},
  fonttitle=\bfseries,
  fontupper=\small, 
  enhanced, 
  arc=2mm,
  boxrule=0.5pt,
  label={alg:ref-target}
]
\begin{algorithm2e}[H]
\DontPrintSemicolon
\label{alg:ref_target_selection}
\KwIn{Weighted, directed graph $G=(V,E,w)$.}
\KwData{Fraction of target heads to predict  $f$, minimum reference heads to be a valid target $m$.}
\KwResult{Target heads $T$, $r^{2}$ edge strength $\tau$}

\BlankLine
\textbf{low} $\gets 0$, \textbf{high} $\gets 1$ \tcp*[r]{binary-search bounds}

\While{not converged}{
    \tcp{1.  Prune the graph to significant edges}
    $\tau \gets (\textbf{low} + \textbf{high})/2$\;
    
    $G^* \gets$ subgraph of $G$ with edges $\ge \tau$\;

    \BlankLine
    \tcp{2.  Greedily select a set of target heads}
    \Indp
        $T \gets \emptyset$\;
        
        Candidates $\gets \{v \in V(G^*) \mid \text{indegree}(v) \ge m\}$\;
        
        \While{a valid candidate can be added to $T$}{
            \tcc{A candidate $c$ is \emph{valid} if every node in $T\cup\{c\}$\\
                 \hspace*{1.8em}is fed by at least $m$ non-target nodes.}
            $c \gets$ Valid Candidate $\setminus T$\;
            
            $T \gets T \cup \{c\}$\;
        }
    \Indm

    \BlankLine
    \tcp{3.  Adjust binary-search window}
    \eIf{$|T| < f \cdot |V|$}{
        \textbf{high} $\gets \tau$\;   \tcp*[r]{ too few targets}
    }{
        \textbf{low} $\gets \tau$\;    \tcp*[r]{too many targets}
    }
}
\BlankLine
\end{algorithm2e}
\end{tcolorbox}


%% file: tables/KV_predict_main_results.tex
\begin{table}[b]
  \caption{Evaluation accuracy (\%) across multiple benchmarks with and without 50\% KV prediction.}
  \label{tab:kv-pred-results-main}
  \centering
  \resizebox{\linewidth}{!}{%
  \begin{tabular}{llccccccc}
    \toprule
    \textbf{Model} & \textbf{Method} & \textbf{TruthfulQA} & \textbf{HellaSwag} & \textbf{MMLU Stem} & \textbf{GPQA} & \textbf{Winogrande} & \textbf{Average} \\
    \midrule
        \multirow{2}{*}{Qwen3 32B} 
        & Baseline        & 39.16 & 82.70 & 82.39 & 40.62 & 73.16 & 63.61 \\
        & 50\% KV Pred    & 40.02 & 77.02 & 74.21 & 31.91 & 67.24 & 58.08 \\
    \midrule
    \multirow{2}{*}{Falcon3 10B} 
        & Baseline        & 37.05 & 80.95 & 68.47 & 34.82 & 71.19 & 58.50 \\
        & 50\% KV Pred    & 35.00 & 74.23 & 61.90 & 27.67 & 71.42 & 54.04 \\
    \midrule
    \multirow{2}{*}{Llama3.1 8B} 
        & Baseline        & 37.08 & 79.28 & 58.70 & 29.68 & 73.87 & 55.72 \\
        & 50\% KV Pred    & 30.35 & 66.53 & 44.53 & 26.34 & 61.40 & 45.83 \\
    \bottomrule
  \end{tabular}
  }
\end{table}

%% file: tables/key_and_value_prediction.tex
\begin{table}[ht] 
  \centering
  \caption{Model performance across different tasks for \textbf{50\% key and value} predictions on Llama3.1 8b instruct}
  \label{tab:key_value_prediction_table}

  \small
  \begin{tabular}{@{} l l
                  S[table-format=2.2]
                  S[table-format=3.3]
                  S[table-format=2.2]
                  S[table-format=2.2] @{}}
    \toprule
    \multirow{2}{*}{Evaluation Domain} &
    \multirow{2}{*}{Task/Metric} & &
    \multicolumn{3}{c}{Calibration Domain} \\
    \cmidrule(lr){4-6}
     & & {Baseline} & {Maths} & {Coding} & {Psychology} \\
     & &            & \text{(GSM8K)}    & \text{(Humaneval)}      & \text{(MMLU Psyc)} \\ 
    \midrule
    \multirow{7}{*}{MATH}
      & GSM8k                     & 75.35 & 71.03 & 67.77 & 62.31 \\
      & Math Intermediate Algebra & 10.74 &  6.97 & 7.41 &  7.64 \\
      & Math Prealgebra           & 41.90 & 36.28 & 32.60 & 34.32 \\
      & Math Precalculus          & 11.42 &  6.59 & 6.77 & 7.14 \\
      & Math Algebra              & 42.40 & 26.11 & 22.41 & 20.97 \\
      & MMLU Pro Math             & 38.00 & 33.30 & 30.94 & 30.71 \\
      & \textbf{Average}          & \bfseries 36.77 & \bfseries 30.04 & \bfseries 27.98 & \bfseries 27.18 \\
    \midrule
    \multirow{2}{*}{CODING}
      & Humaneval                 & 58.35 & 20.12 & 39.63 & 12.8 \\
      & \textbf{Average}          & \bfseries 58.35 & \bfseries 20.12 & \bfseries 39.63 & \bfseries 12.8 \\
    \midrule
    \multirow{5}{*}{OTHERS}
      & MMLU Pro CS               & 39.02 & 38.78 & 34.63 & 37.07 \\
      & MMLU Pro Law              & 22.25 & 17.25 & 20.43 & 24.88 \\
      & MMLU Pro Health           & 47.92 & 36.18 & 37.04 & 46.45 \\
      & MMLU Pro Psychology       & 57.51 & 52.88 & 52.25 & 57.39 \\
      & \textbf{Average}          & \bfseries 41.68 & \bfseries 36.27 & \bfseries 36.087 & \bfseries 41.44 \\
    \bottomrule
  \end{tabular}
\end{table}

%% file: sections/relatedwork.tex
\section{Related Work}
\label{sec:related_work}

Our research builds upon two main areas: the interpretability of attention head redundancy and techniques for KV cache compression. Prior work in mechanistic interpretability has extensively studied attention head functions, revealing specialization and considerable redundancy \citep{clark2019bertattention, voita2019analyzing, olsson2022induction}. This has led to efficiency methods like head pruning based on importance scores or attention map correlations \citep{michel2019sixteen, agarwal2024chai}, and architectural modifications like GQA that share K/V projections \citep{ainslie2023gqa}. These approaches typically assess redundancy via downstream task impact or attention map similarity. Our work differs by focusing on a finer-grained phenomenon: the linear predictability of internal Key, Value, and Query vector activations themselves. We show that these raw activations often exhibit strong linear correlations, a distinct type of learned structure not directly captured by prior methods.

Complementary research on KV cache optimization aims to reduce its memory footprint during inference. Strategies include imposing structural constraints like low-rank projections \citep{wang2020linformer} or cross-layer factorization (\textsc{xKV}; \citealp{chang2025xkv}), and dictionary-based methods like \textsc{Lexico} \citep{kim2024lexico} that use sparse coding. Most closely related to our work, \textsc{AQUA-KV} \citep{shutova2025cache} learns compact \emph{layer-wise} linear predictors to improve KV-cache \emph{quantization}---reducing residual magnitudes and storing a quantized residual for \emph{every} head. Our method differs in goal and mechanism: we exploit inter-head linearity to \emph{eliminate} storing a subset of heads' KV caches entirely, retaining full-precision states only for reference heads and reconstructing the rest on-the-fly with no quantization. Also we only use a maximum of 5 heads unlike AQUA-KV which uses all key and value heads from the current layer and previous layer respectively. Beyond the compression scheme, we also characterize the linear predictability phenomenon itself---its emergence during pretraining and theoretical underpinnings---contributions orthogonal to \textsc{AQUA-KV}'s quantization focus.

%% file: sections/conclusion.tex
\vspace{-1em}
\section{Discussion and Final Remarks}
\label{sec:discussion}

This paper identifies and characterizes inter-head linear predictability as an emergent structural property of pretrained Transformers. Through experiments on four model families, we show that individual attention heads' Key, Value, and Query activations can be accurately reconstructed from 2--5 peer heads---a regularity absent at random initialization and developing steadily during pretraining. Analysis of OLMo-2 checkpoints reveals that intra-layer prediction strength and projection-subspace overlap grow monotonically, with Key subspaces aligning fastest, offering a mechanistic account of the phenomenon. A theoretical lower bound (Theorem~\ref{thm:main-lowerbound}) confirms that random initialization produces high MSE with overwhelming probability.

As a downstream application, we compress the KV cache by eliminating storage of approximately half the heads and reconstructing them on-the-fly. This yields $2\times$ memory reduction with accuracy trade-offs that vary by model: Qwen3-32B and Falcon3-10B incur average drops of 5.5\,pp and 4.5\,pp respectively, while LLaMA-3.1 8B shows a larger degradation of 9.9\,pp. We further find that Key compression is substantially more benign than Value compression, consistent with the asymmetric subspace overlap growth observed during training.

Several avenues remain open. Identifying the precise pretraining dynamics and architectural factors that drive inter-head alignment could yield fundamental insights into LLM learning. More sophisticated reconstruction techniques---nonlinear maps, dynamic reference selection, or hybrid quantization---may improve compression fidelity, particularly for models where linear prediction degrades more significantly. Beyond compression, the discovered subspace structure may inform head pruning, model distillation, and mechanistic interpretability.

\textbf{Limitations.} While we demonstrate the phenomenon, the exact causal mechanisms behind its emergence during pretraining are not well-understood and remain an open question. Our KV cache compression method, though effective, introduces a calibration step and a minor inference-time reconstruction latency; moreover, the degree of linear predictability and thus compression benefits can vary with the specific model, task, and data.






%% file: appendix/add_kvcache.tex
\section{Results for Transferability of Linear Relations Across Tasks}\label{app:kv_cache}

We investigate how replacing half of the \emph{key} (\(K\)) or \emph{value} (\(V\)) states with linear predictions impacts downstream accuracy.  In particular, we ask: \emph{Are keys or values more resilient to approximation, and does the calibration domain of the linear projector matter?}

Tables~ \ref{tab:key_value_prediction_table}, \ref{tab:key_prediction_transferable_table}, and \ref{tab:value_prediction_transefrable_table} disentangle three settings: joint prediction of both streams (Table~3), prediction of only \emph{keys} (Table~4), and prediction of only \emph{values} (Table~5).  
A clear hierarchy emerges: \textbf{key compression is    less harmful than value compression, and the joint setting is the most disruptive}.  
For the maths suite, the average score declines from \(36.77\,\%\) to \(34.80\,\%\) (\(-1.97\) pp) when half of the keys are reconstructed, but drops to \(32.47\,\%\) (\(-4.30\) pp) for value reconstruction and to \(30.04\,\%\) (\(-6.73\) pp) when both keys and values are predicted.

\input{tables/key_prediction}

\input{tables/value_prediction}

Calibrating on a matching domain helps reduce the accuracy drop, but value states remain much more fragile than keys.  Math‐calibrated projectors, for instance, preserve GSM8K accuracy under key compression (\(74.60\,\%\) vs.\ \(75.35\,\%\) baseline) yet still lose nearly \(3\) pp when compressing values. Surprisingly, even a projector trained on a seemingly unrelated domain—Psychology—can compress keys while still preserving roughly 92\% of the baseline accuracy on math benchmarks; when the same cross-domain projector is applied to the values, however, performance deteriorates far more sharply.

These results suggest that keys contain more linearly redundant or similar information than values across tasks from different domains.  Predicting them offers significant memory savings with only modest accuracy degradation.  Values are more task‐specific; compressing them—even partially—risks larger performance drops, and simultaneous K\,+\,V compression compounds these errors.  These findings advocate a hybrid strategy: aggressively compress keys while treating values conservatively, or calibrating them with domain‐specific regressors, to balance memory efficiency with robust task accuracy.

%% file: tables/key_prediction.tex
\begin{table}[h] 
  \centering
  \caption{Model performance across different tasks for \textbf{50\% key} predictions on Llama3.1 8b instruct}
  \label{tab:key_prediction_transferable_table}


  \small
  \begin{tabular}{@{} l l
                  S[table-format=2.2]
                  S[table-format=3.3]
                  S[table-format=2.2]
                  S[table-format=2.2] @{}}
    \toprule
    \multirow{2}{*}{Evaluation Domain} &
    \multirow{2}{*}{Task/Metric} & &
    \multicolumn{3}{c}{Calibration Domain} \\
    \cmidrule(lr){4-6}
     & & {Baseline} & {Maths} & {Coding} & {Psychology} \\
     & &            & \text{(GSM8K)}    & \text{(Humaneval)}      & \text{(MMLU Psyc)} \\ 
    \midrule
    \multirow{7}{*}{MATH}
      & GSM8k                     & 75.35 & 74.60 & 74.60 & 74.22 \\
      & Math Intermediate Algebra & 10.74 &  9.85 & 11.18 &  9.30 \\
      & Math Prealgebra           & 41.90 & 40.64 & 40.18 & 39.83 \\
      & Math Precalculus          & 11.42 &  9.70 & 10.99 & 10.07 \\
      & Math Algebra              & 42.40 & 37.23 & 37.07 & 34.50 \\
      & MMLU Pro Math             & 38.00 & 36.78 & 35.16 & 34.86 \\
      & \textbf{Average}          & \bfseries 36.77 & \bfseries 34.80 & \bfseries 34.86 & \bfseries 33.79 \\
    \midrule
    \multirow{2}{*}{CODING}
      & Humaneval                 & 58.35 & 51.82 & 45.73 & 46.34 \\
      & \textbf{Average}          & \bfseries 58.35 & \bfseries 51.82 & \bfseries 45.73 & \bfseries 46.34 \\
    \midrule
    \multirow{5}{*}{OTHERS}
      & MMLU Pro CS               & 39.02 & 38.04 & 37.32 & 36.00 \\
      & MMLU Pro Law              & 22.25 & 20.16 & 22.43 & 20.79 \\
      & MMLU Pro Health           & 47.92 & 47.31 & 46.33 & 47.31 \\
      & MMLU Pro Psychology       & 57.51 & 56.89 & 56.01 & 56.01 \\
      & \textbf{Average}          & \bfseries 41.68 & \bfseries 40.60 & \bfseries 40.52 & \bfseries 40.00 \\
    \bottomrule
  \end{tabular}
\end{table}

%% file: tables/value_prediction.tex
\begin{table}[h] 
  \centering
  \caption{Model performance across different tasks for \textbf{50\% value} predictions on Llama3.1 8b instruct}
  \label{tab:value_prediction_transefrable_table}

  \small
  \begin{tabular}{@{} l l
                  S[table-format=2.2]
                  S[table-format=3.3]
                  S[table-format=2.2]
                  S[table-format=2.2] @{}}
    \toprule
    \multirow{2}{*}{Evaluation Domain} &
    \multirow{2}{*}{Task/Metric} & &
    \multicolumn{3}{c}{Calibration Domain} \\
    \cmidrule(lr){4-6}
     & & {Baseline} & {Maths} & {Coding} & {Psychology} \\
     & &            & \text{(GSM8K)}    & \text{(Humaneval)}      & \text{(MMLU Psyc)} \\ 
    \midrule
    \multirow{7}{*}{MATH}
      & Gsm8k                     & 75.35 & 72.78 & 69.21 & 68.61 \\
      & Math Intermediate Algebra & 10.74 &  9.41 &  9.30 &  8.41 \\
      & Math Prealgebra           & 41.90 & 40.64 & 34.90 & 42.13 \\
      & Math Precalculus          & 11.42 &  8.24 &  7.69 &  7.87 \\
      & Math Algebra              & 42.40 & 30.41 & 27.96 & 28.47 \\
      & MMLU Pro Math             & 38.00 & 33.38 & 33.75 & 33.67 \\
      & \textbf{Average}          & \bfseries 36.77 & \bfseries 32.47 & \bfseries 30.46 & \bfseries 31.526 \\
    \midrule
    \multirow{2}{*}{CODING}
      & Humaneval                 & 58.35 & 42.68 & 49.39 & 35.36 \\
      & \textbf{Average}          & \bfseries 58.35 & \bfseries 42.68 & \bfseries 49.39 & \bfseries 35.36 \\
    \midrule
    \multirow{5}{*}{OTHERS}
      & MMLU Pro CS               & 39.02 & 36.58 & 36.34 & 36.82 \\
      & MMLU Pro Law              & 22.25 & 19.61 & 18.25 & 22.25 \\
      & MMLU Pro Health           & 47.92 & 40.46 & 37.16 & 47.67 \\
      & MMLU Pro Psychology       & 57.51 & 52.38 & 51.87 & 57.89 \\
      & \textbf{Average}          & \bfseries 41.68 & \bfseries 37.25 & \bfseries 35.905 & \bfseries 41.157 \\
    \bottomrule
  \end{tabular}
\end{table}

%% file: appendix/add_lin_pred.tex
\section{Additional Results on Linear Predictability of Heads}\label{app:add_lin}
\noindent
\textbf{Trends in inter-layer linear connection.} In this section, we quantify how the strength and prevalence of dominant linear links among attention heads vary with their inter‐layer distance.  To do so, we first construct a directed graph \(G\) whose nodes correspond to heads (parsed as \(\text{(layer, head)}\) pairs) and whose edge weights are the \(R^2\) scores from pairwise OLS regressions of key–state activations. Edges are split into “close” and “far” windows using two distance thresholds (\(\le2\) vs.\ \(\ge3\) layers, and \(\le4\) vs.\ \(\ge5\) layers). For each proximity category, we aggregate raw edge weights and normalize them to yield the percentage contribution of that category within its window (``\% Conn'') and compute the corresponding mean \(R^2\).  Table \ref{tab:proximity_density_callout} then reports these two metrics for both 2- and 4-layer windows, showing that nearby links contribute over 53\% of each window’s normalized weight and consistently exhibit higher predictability than distant links. Far-layer connections remain substantial (\(\approx 46\% \) ) yet are consistently weaker, underscoring a strong local alignment in the attention subspace despite the presence of non-local dependencies.
\input{tables/near_far_connectivity}

\textbf{Distributions of linear predictability across different tasks and models.}
In this section, we further demonstrate that inter-head linear dependencies persist across both downstream benchmarks and pretraining corpora.  We consider two representative settings: (i) LLaMA-3.1 8B evaluated on GSM8k, and (ii) Olmo-2 7B evaluated on its own pretraining data.  Figures \ref{fig:r2_cdf_kde_llama3_gsm8k} and \ref{fig:r2_cdf_kde_olmo2_train} respectively plot, for each projection type (Key left, Query center, Value right), (a) the cumulative distribution function (CDF) of \(R^2\) and (b) the kernel-density estimate (KDE) of the same, as the number of reference heads \(N\) varies from 1 to 5.  

Across both models and datasets, increasing \(N\) yields a clear rightward shift in the CDFs and a corresponding movement of KDE peaks toward higher \(R^2\), with the largest marginal gain between \(N=1\) and \(N=2\) and diminishing returns beyond \(N=4\). Finally, GSM8k exhibits slightly stronger linear alignment (median \(R^2\approx0.85\) for Keys at \(N=5\)) compared to Olmo-2’s raw training data, indicating that downstream tasks can amplify intrinsic linear structure in attention head activations.

\begin{figure}[ht]
  \centering
  \begin{subfigure}[t]{0.31\linewidth}
    \centering
    \includegraphics[width=\linewidth]{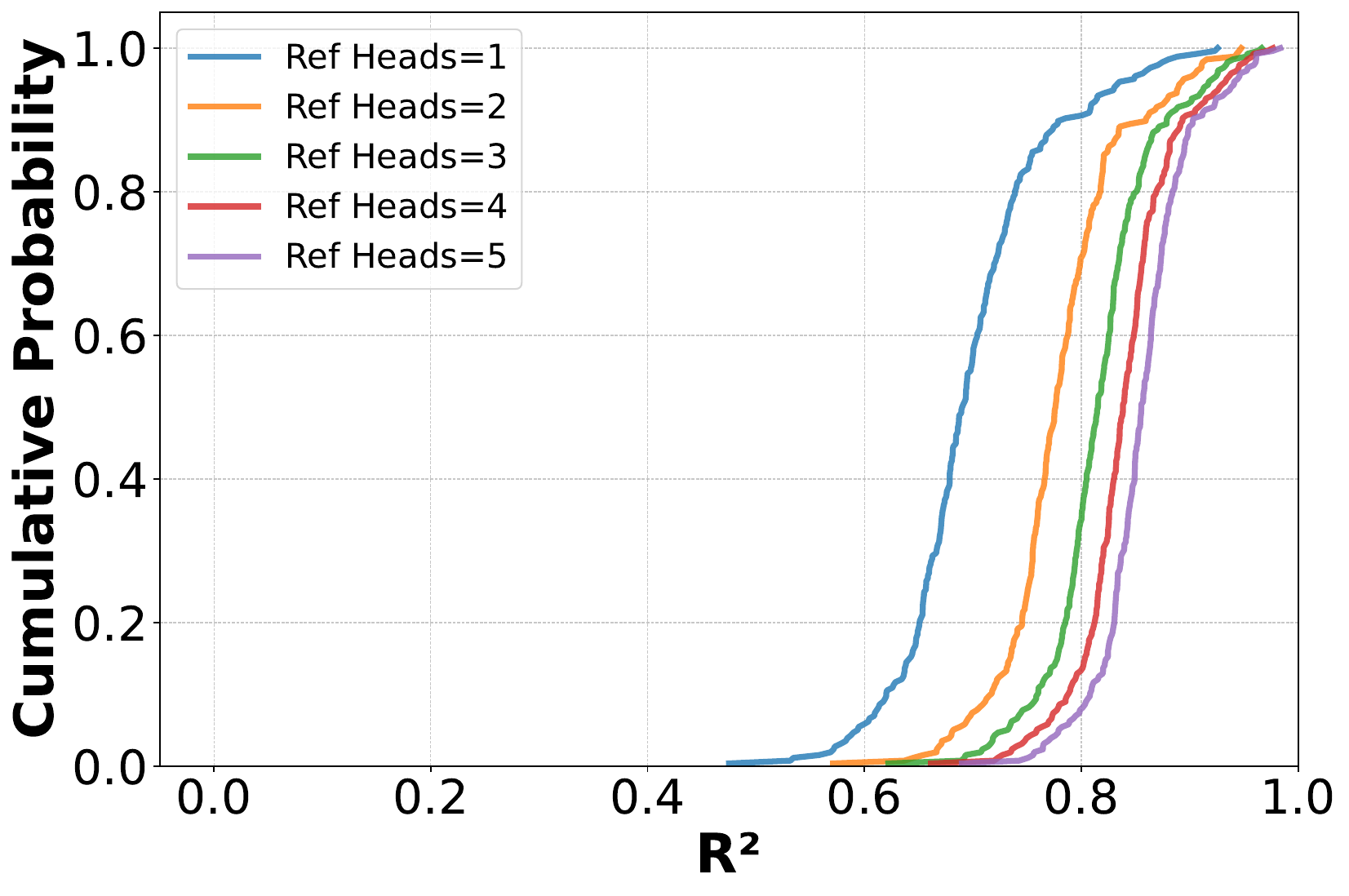}
    \subcaption{Key (K) on GSM8k}
    \label{subfig:cdf_key_gsm8k}
  \end{subfigure}\hfill
  \begin{subfigure}[t]{0.31\linewidth}
    \centering
    \includegraphics[width=\linewidth]{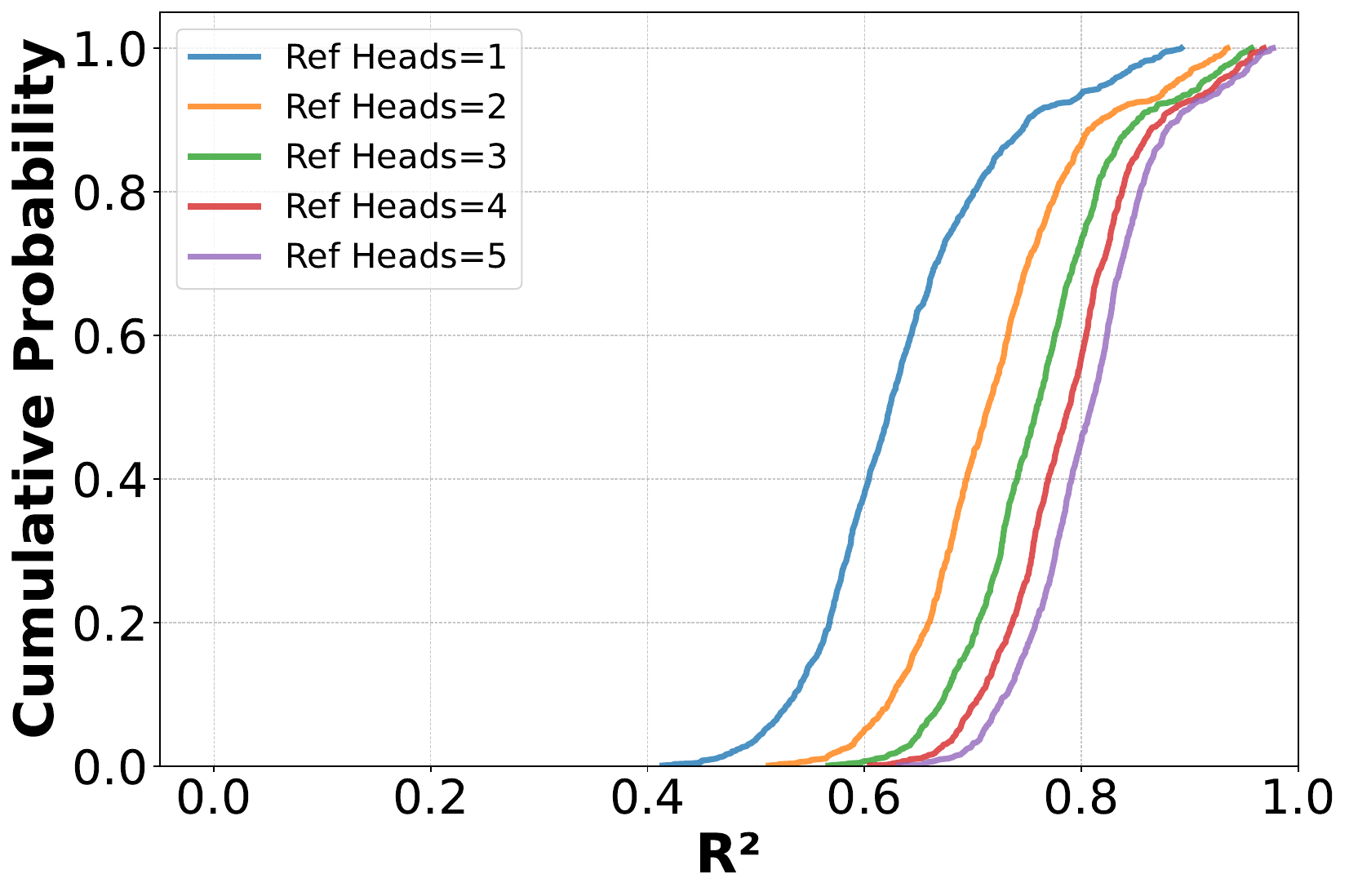}
    \subcaption{Query (Q) on GSM8k}
    \label{subfig:cdf_query_gsm8k}
  \end{subfigure}\hfill
  \begin{subfigure}[t]{0.31\linewidth}
    \centering
    \includegraphics[width=\linewidth]{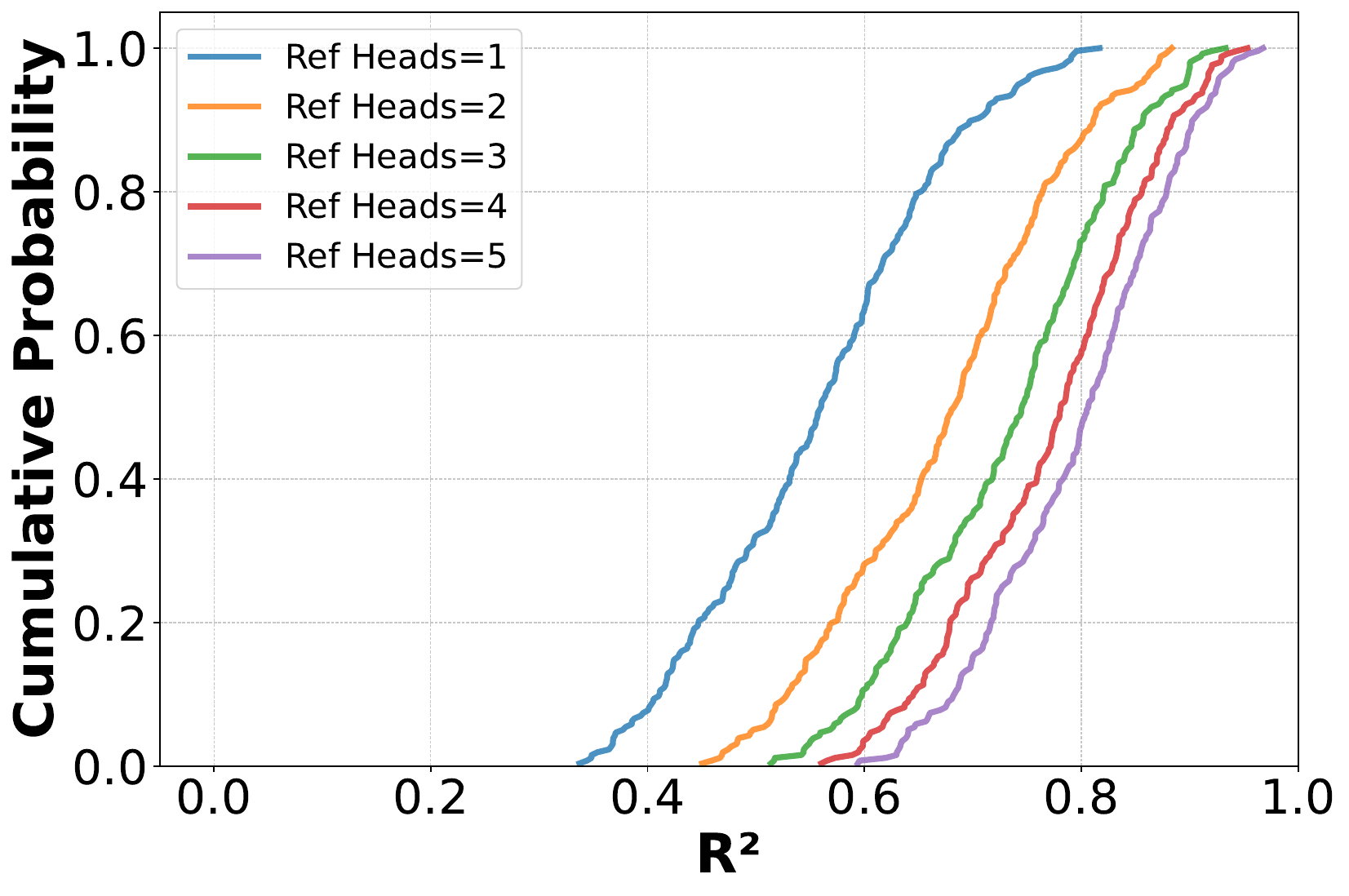}
    \subcaption{Value (V) on GSM8k}
    \label{subfig:cdf_value_gsm8k}
  \end{subfigure}

  \vspace{0.6em}

  \begin{subfigure}[t]{0.31\linewidth}
    \centering
    \includegraphics[width=\linewidth]{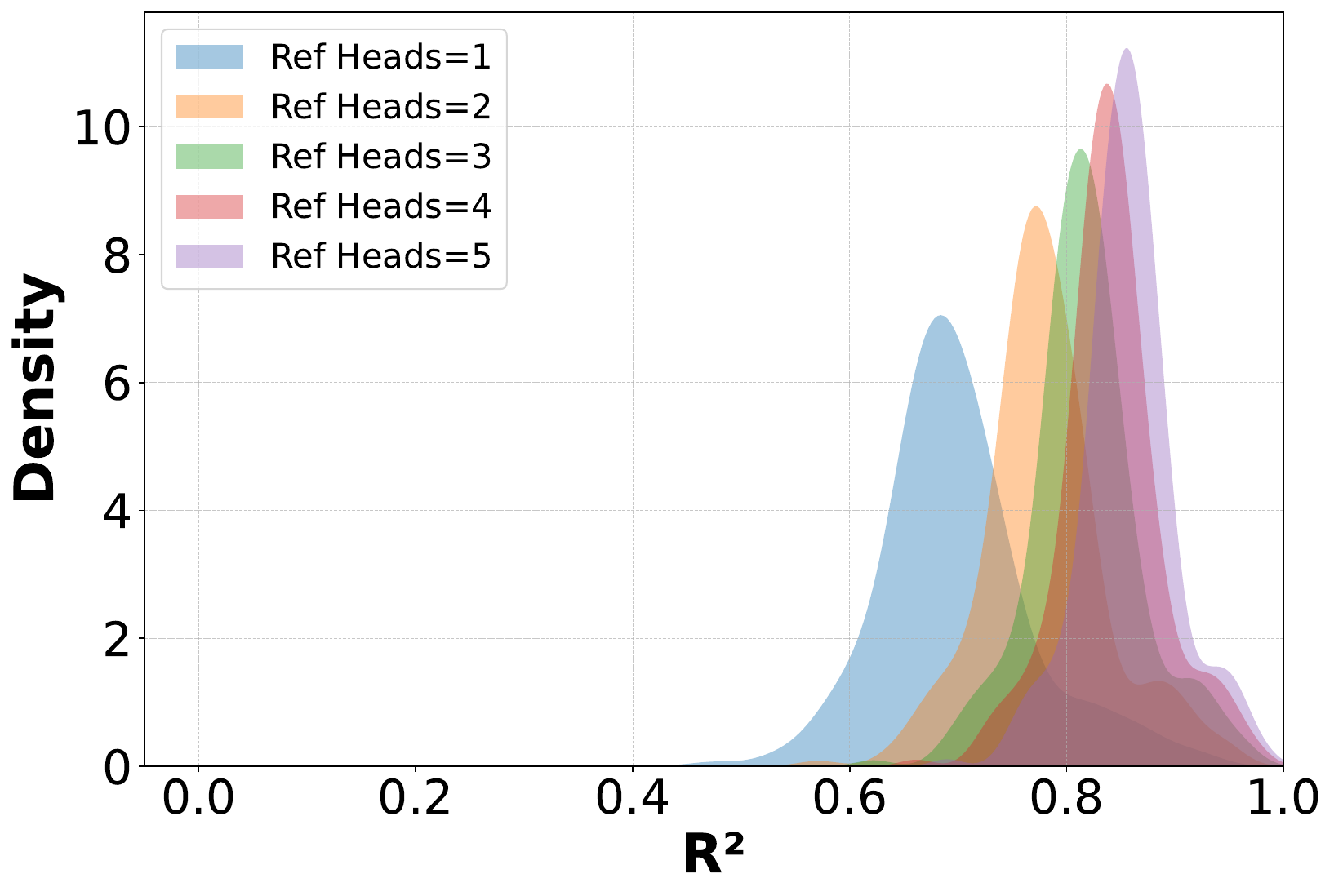}
    \subcaption{Key (K) on GSM8k}
    \label{subfig:kde_key_gsm8k}
  \end{subfigure}\hfill
  \begin{subfigure}[t]{0.31\linewidth}
    \centering
    \includegraphics[width=\linewidth]{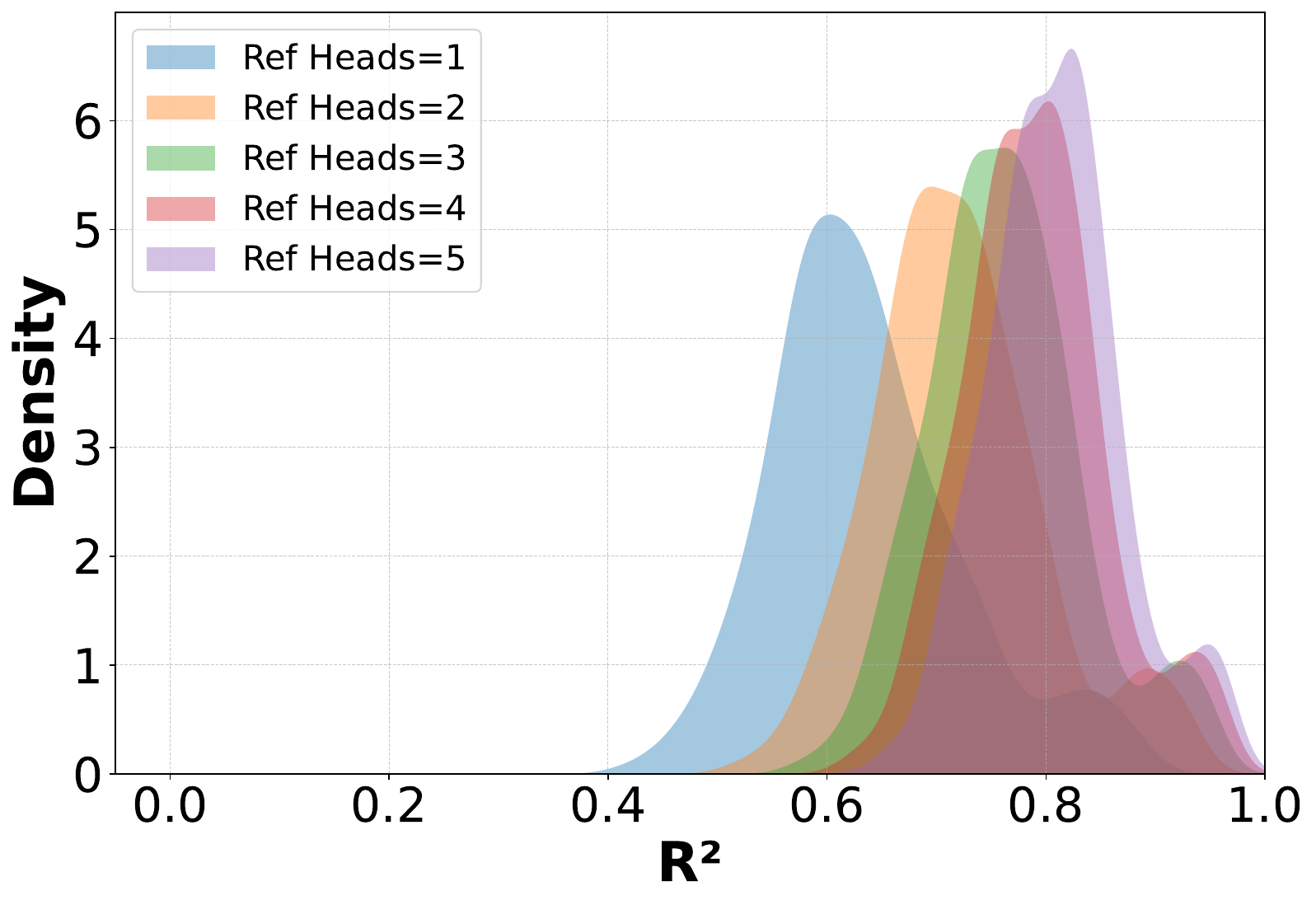}
    \subcaption{Query (Q) on GSM8k}
    \label{subfig:kde_query_gsm8k}
  \end{subfigure}\hfill
  \begin{subfigure}[t]{0.31\linewidth}
    \centering
    \includegraphics[width=\linewidth]{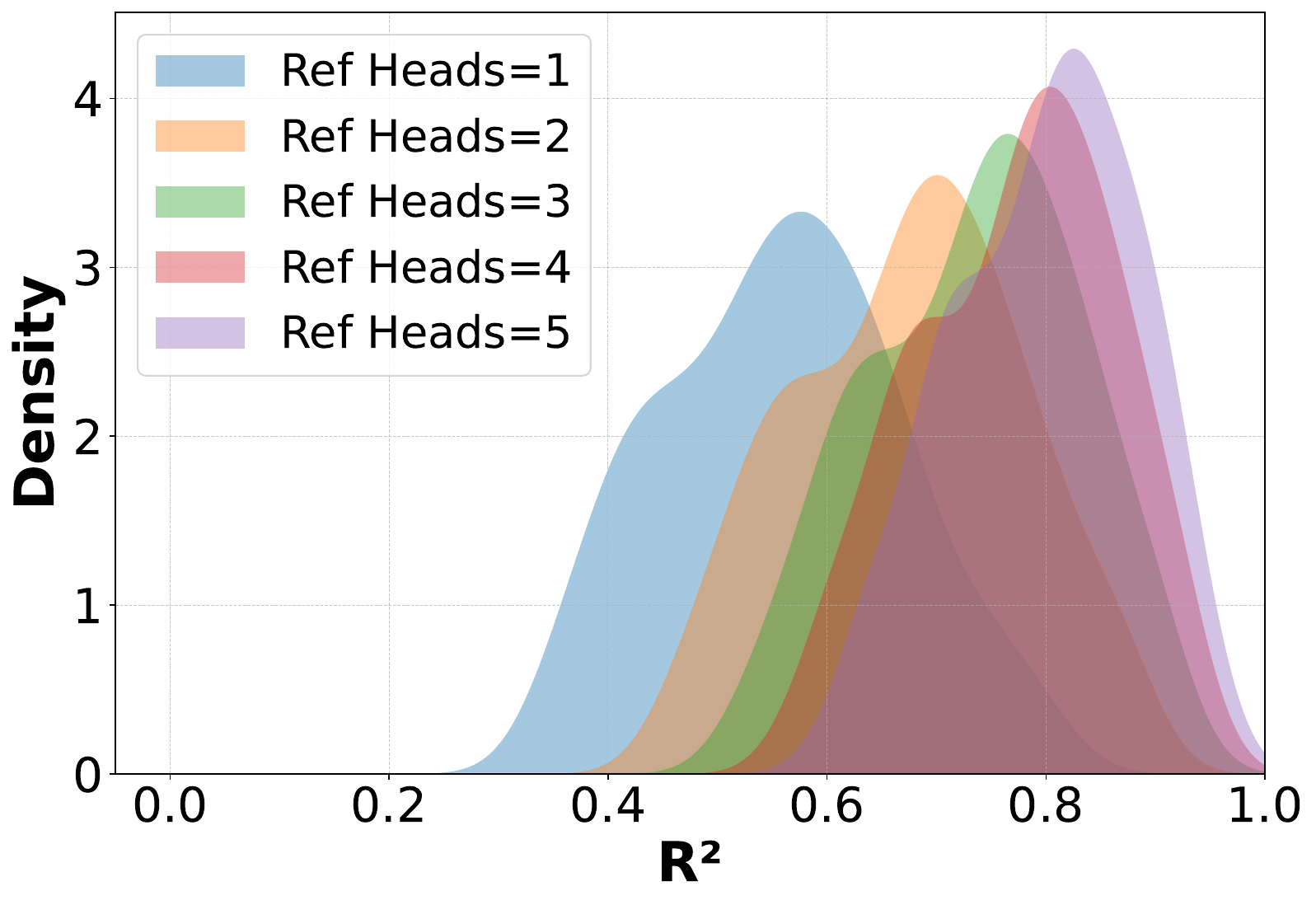}
    \subcaption{Value (V) on GSM8k}
    \label{subfig:kde_value_gsm8k}
  \end{subfigure}

  \caption{\textbf{GSM8k – Distribution of linear predictability (\(R^2\)) for LLaMA-3.1 8B.}  
           Top row: cumulative distribution functions (CDFs); bottom row: kernel-density estimates (KDEs).  
           Ordering (left→right): Key, Query, Value.}
  \label{fig:r2_cdf_kde_llama3_gsm8k}
\end{figure}

\begin{figure}[ht]
  \centering
  \begin{subfigure}[t]{0.31\linewidth}
    \centering
    \includegraphics[width=\linewidth]{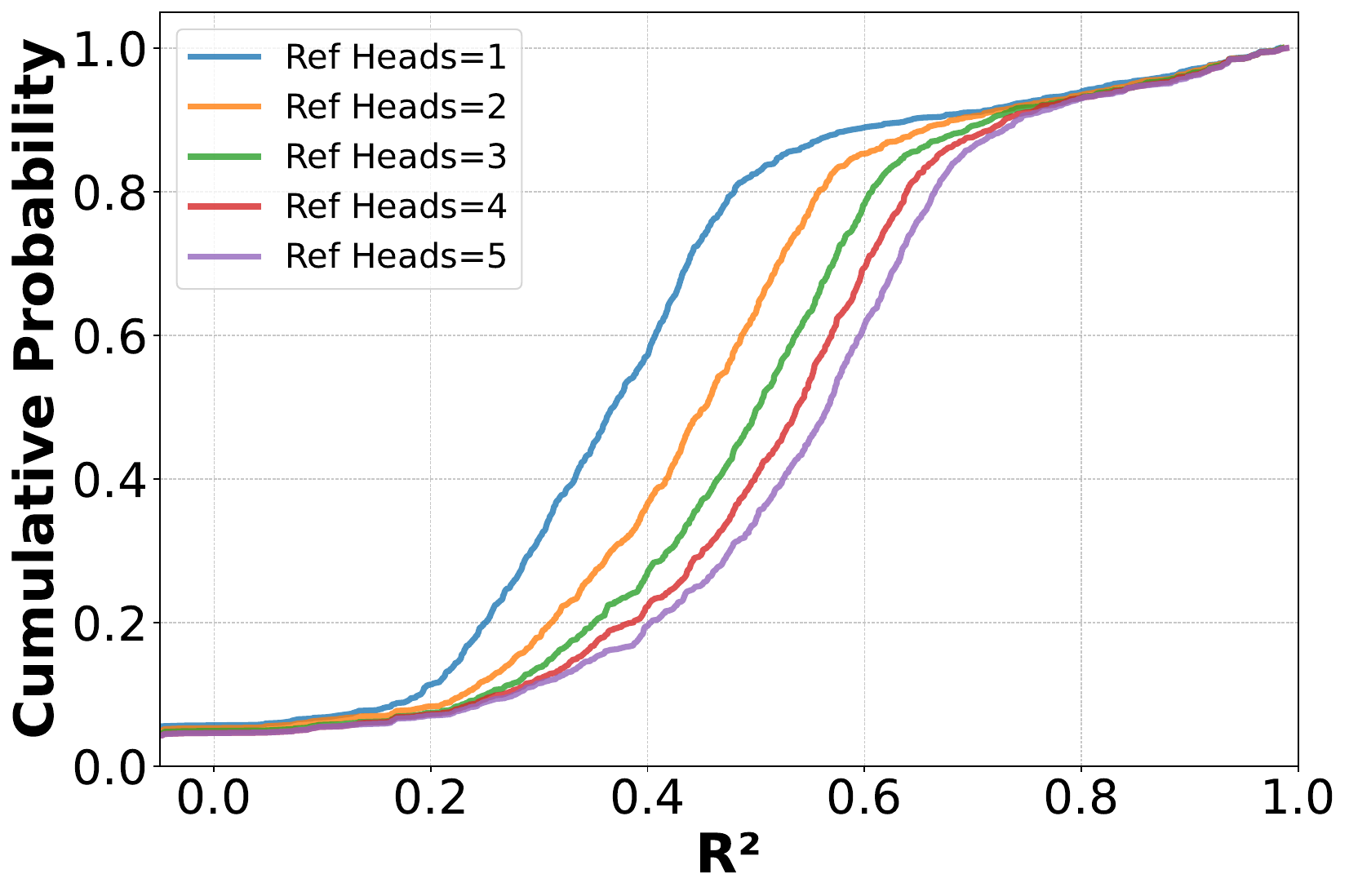}
    \subcaption{Key (K) on Olmo Train}
    \label{subfig:cdf_key_olmo2}
  \end{subfigure}\hfill
  \begin{subfigure}[t]{0.31\linewidth}
    \centering
    \includegraphics[width=\linewidth]{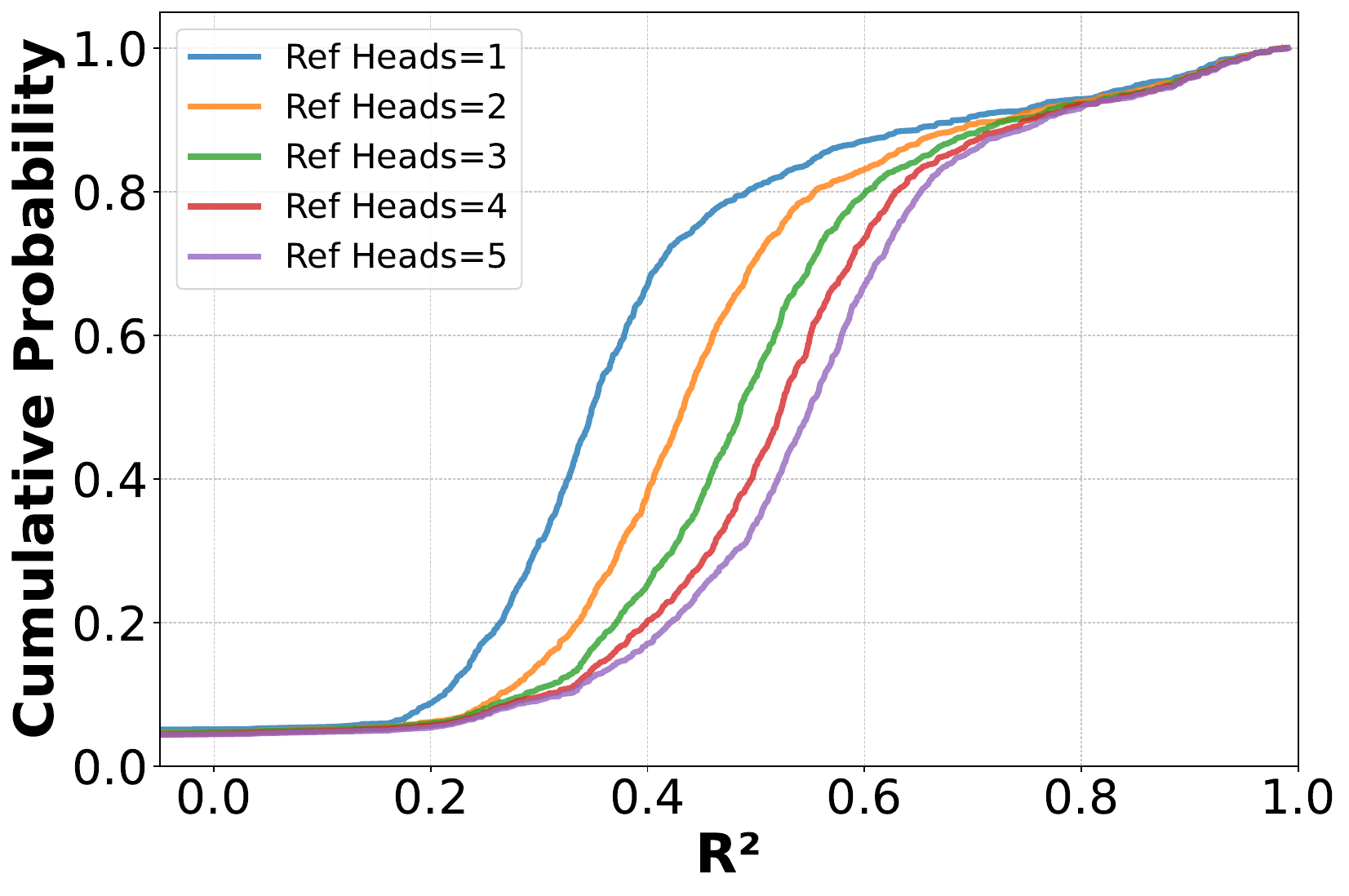}
    \subcaption{Query (Q) on Olmo Train}
    \label{subfig:cdf_query_olmo2}
  \end{subfigure}\hfill
  \begin{subfigure}[t]{0.31\linewidth}
    \centering
    \includegraphics[width=\linewidth]{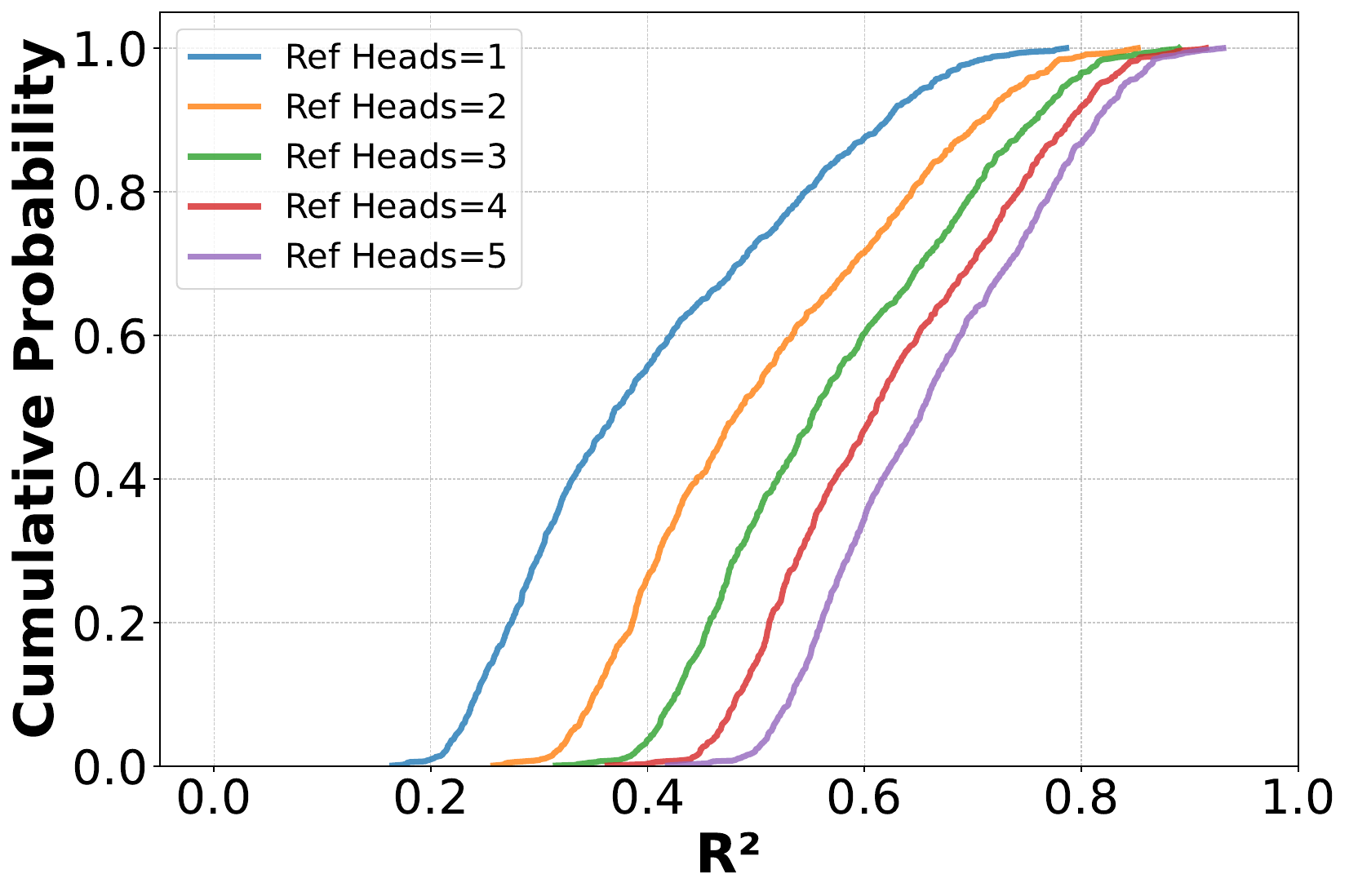}
    \subcaption{Value (V) on Olmo Train}
    \label{subfig:cdf_value_olmo2}
  \end{subfigure}

  \vspace{0.6em}

  \begin{subfigure}[t]{0.31\linewidth}
    \centering
    \includegraphics[width=\linewidth]{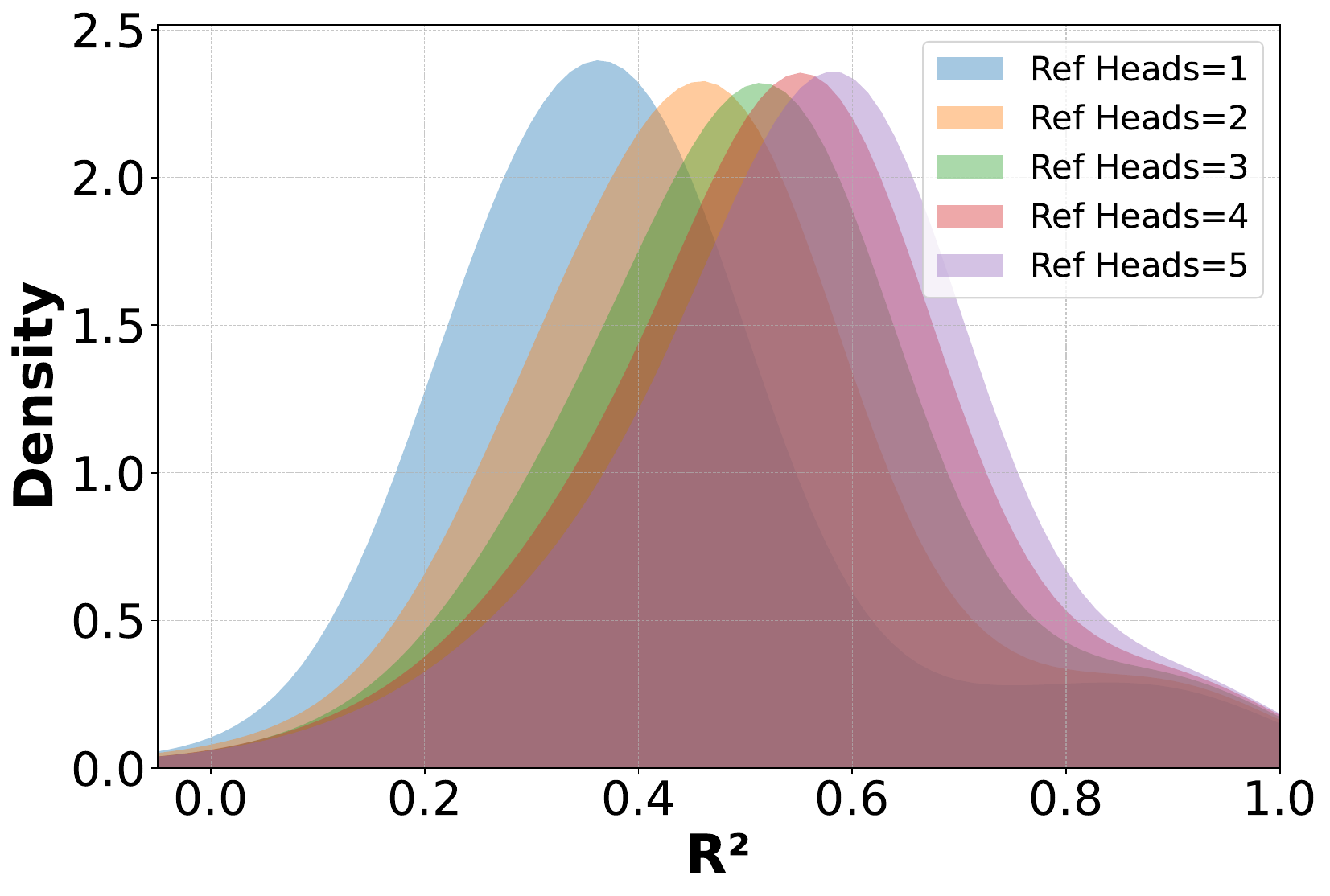}
    \subcaption{Key (K) on Olmo Train}
    \label{subfig:kde_key_olmo2}
  \end{subfigure}\hfill
  \begin{subfigure}[t]{0.31\linewidth}
    \centering
    \includegraphics[width=\linewidth]{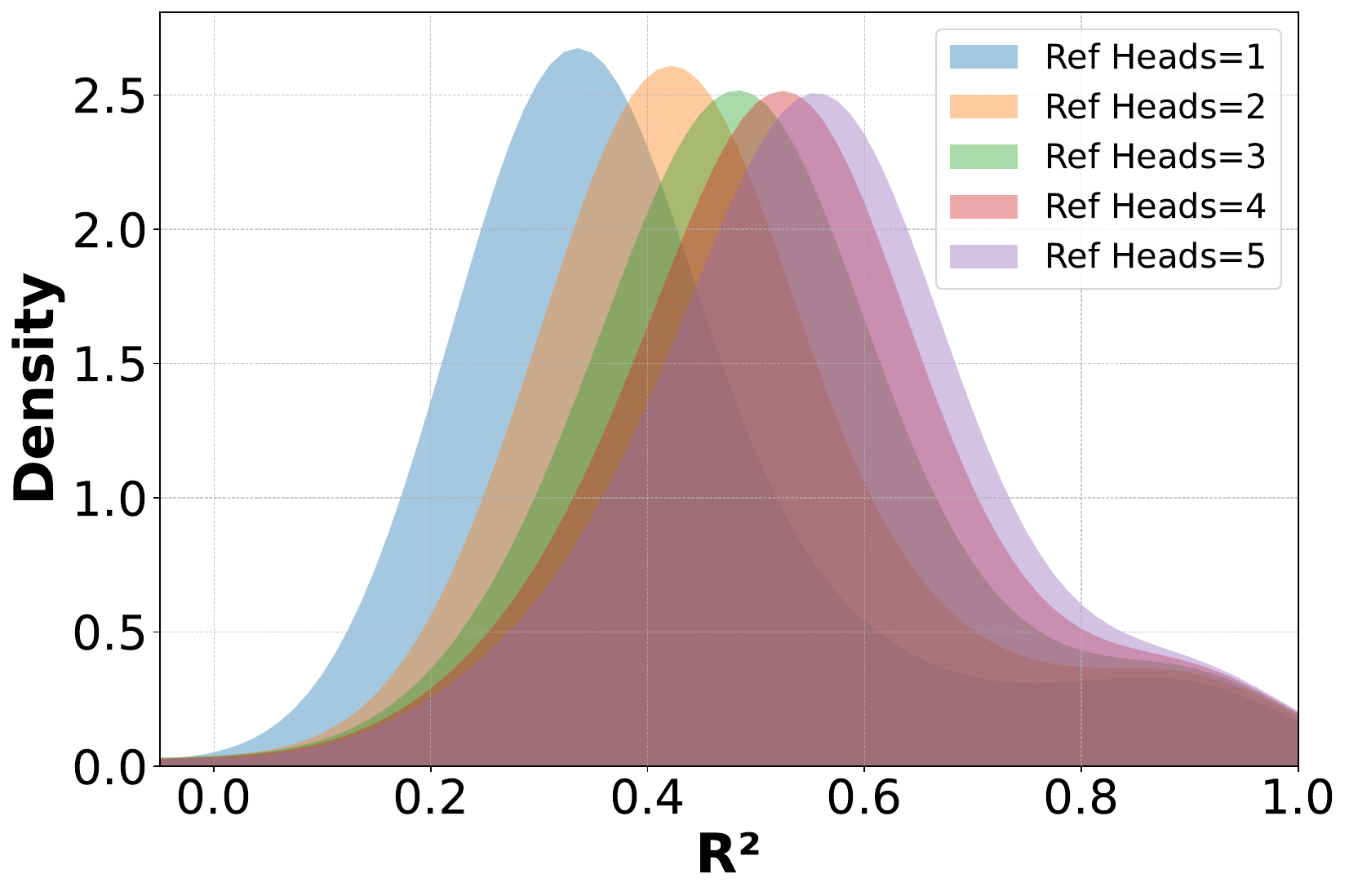}
    \subcaption{Query (Q) on Olmo Train}
    \label{subfig:kde_query_olmo2}
  \end{subfigure}\hfill
  \begin{subfigure}[t]{0.31\linewidth}
    \centering
    \includegraphics[width=\linewidth]{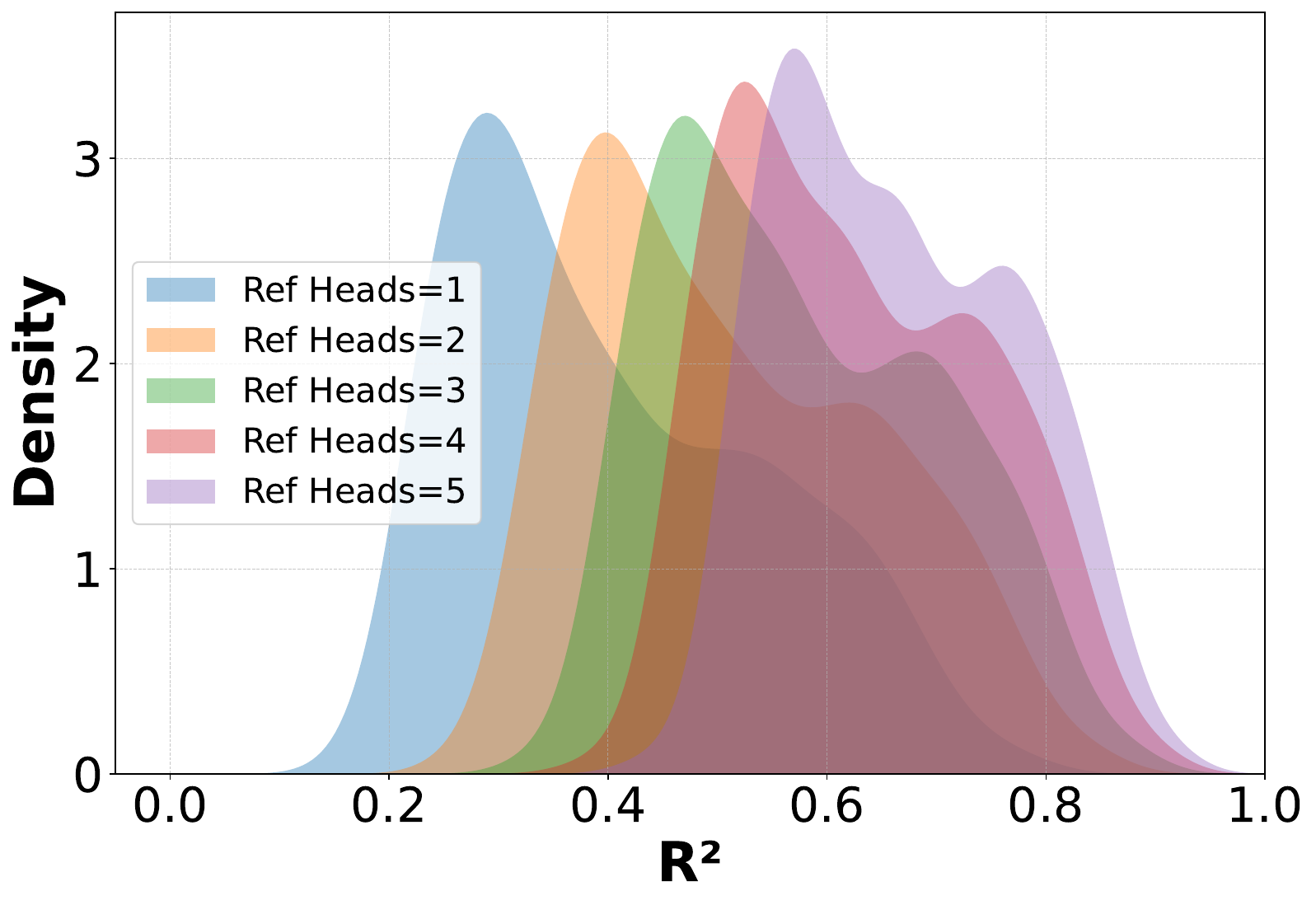}
    \subcaption{Value (V) on Olmo Train}
    \label{subfig:kde_value_olmo2}
  \end{subfigure}

  \caption{\textbf{Olmo 2 7B – Distribution of linear predictability (\(R^2\)) on its pre-training corpus.}  
           Top row: cumulative distribution functions (CDFs).  
           Bottom row: kernel-density estimates (KDEs).  
           Left→right in each row: Key, Query, Value.}
  \label{fig:r2_cdf_kde_olmo2_train}
\end{figure}

\textbf{Intensity of linear relations across layers.} As shown in Figure~\ref{fig:heatmap_linear_predictability} for LLaMA-3.1 8B on C4, the \(R^2\) for each attention head—when linearly predicted from its top reference heads—varies in intensity across layers. While we observe higher linear redundancy in some early layers and lower redundancy in deeper layers for this specific model and dataset, these patterns should not be taken as universally representative. Different architectures and downstream tasks may exhibit distinct layer‐wise profiles of predictability; the only consistent finding is that the density of linear predictability fluctuates across depth of the model.

\begin{figure}[t]
  \centering
  \begin{minipage}[t]{0.44\linewidth}
    \centering
    \includegraphics[width=\linewidth]{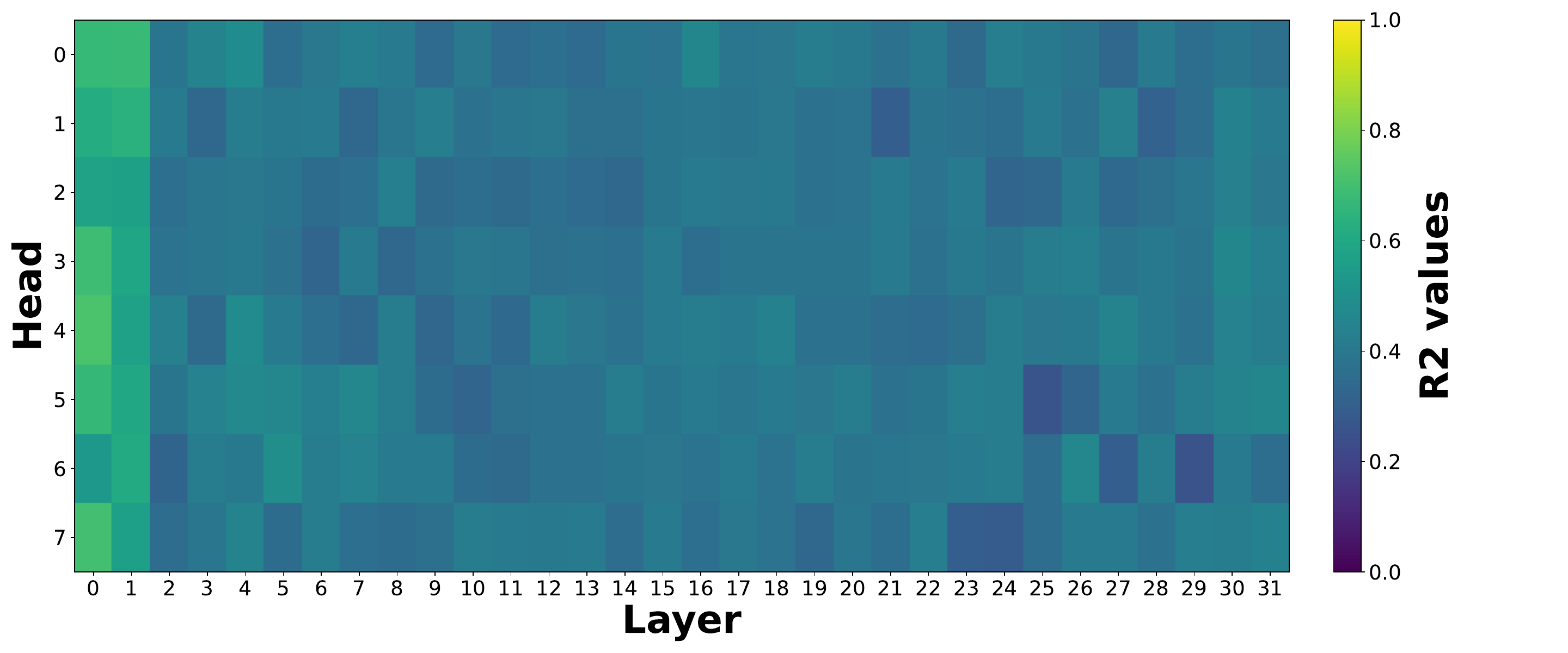}
    \label{fig:heatmap_keys}
  \end{minipage}
  \hfill
  \begin{minipage}[t]{0.44\linewidth}
    \centering
    \includegraphics[width=\linewidth]{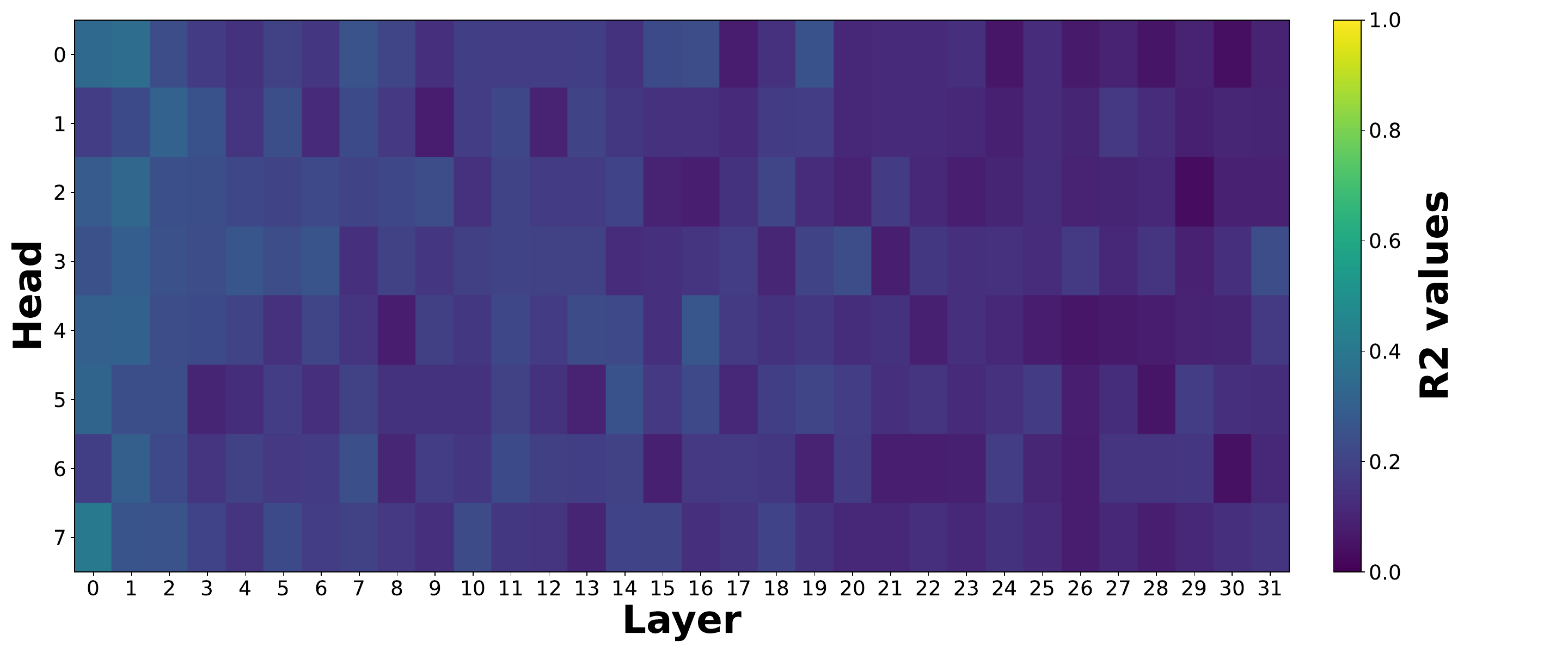}
    \label{fig:heatmap_values}
  \end{minipage}
  \caption{ Heatmap of head‐level linear predictability across transformer layers of Llama3.1 8B on C4. Each row corresponds to one attention head and each column to a model layer; the color intensity encodes the coefficient of determination~$R^2$ when reconstructing that head’s activations from its strongest reference heads. The gradient from warm to cool colors illustrates that linear redundancy is highest in early layers and gradually diminishes in deeper layers for Llama3.1 8B. \textbf{Left} key states, \textbf{right} value states.
}
. \vspace{-1.5em}
  \label{fig:heatmap_linear_predictability}
\end{figure}



%% file: tables/near_far_connectivity.tex
\begin{table}[h]
\caption{Local vs.\ distant dominant linear connections. Boldface and arrows highlight the higher density and strength of near-layer links.}
\label{tab:proximity_density_callout}
\centering
\begin{tabular}{lccr}
\toprule
Window & Proximity      & \% Conn.                     & Mean \(R^2\)                 \\
\midrule
2-layer & Near (\(\le2\)) & \textbf{53.6\%} \(\uparrow\) & \textbf{0.2946} \(\uparrow\) \\
        & Far  (\(\ge3\)) & 46.40\% \(\downarrow\)        & 0.2550 \(\downarrow\)        \\
\midrule
4-layer & Near (\(\le4\)) & \textbf{53.59\%} \(\uparrow\) & \textbf{0.2895} \(\uparrow\) \\
        & Far  (\(\ge5\)) & 46.41\% \(\downarrow\)        & 0.2506 \(\downarrow\)        \\
\bottomrule
\end{tabular}
\end{table}

%% file: appendix/lower_bound.tex
\section{Proof of Theorem~\ref{thm:main-lowerbound}}\label{app:proof_lower}

\begin{theorem}[Hanson-Wright inequality {\cite[Thm. 6.2.2]{Vershynin_2018}}]
\label{thm:hansonwright}

Let $A\in\R^{n\times n}$ be arbitrary matrix and let 
\(
X=(X_1,\dots,X_n)^{\top}\in\br^{n}
\)
have independent, mean-zero, sub-Gaussian coordinates.
Then, for every $t\ge0$,
\[
\Pr\Bigl\{\,
      \bigl|\,X^{\top}AX - \E\bigl[X^{\top}AX\bigr]\bigr|
       \ge t
   \Bigr\}
 \le 
2\exp\Bigl[
   -\,c 
   \min\Bigl\{
        \tfrac{t^{2}}{K^{4}\,\|A\|_{F}^{2}},
         \tfrac{t}{K^{2}\,\|A\|_{2}}
   \Bigr\}
\Bigr],
\]
where 
\(
K=\max_{i}\|X_i\|_{\psi_2}
\)
and \(c>0\) is an absolute universal constant.
\end{theorem}

We will now discuss the proof for Theorem~\ref{thm:main-lowerbound}. 
Let $\mA,\mB\in\R^{m\times k}$ have i.i.d.\ $\cN(0,1)$ entries, and assume $k\le m/2$.  Draw 
\[
\vx \sim \cN(\vzero,\mI_m),
\]
independently of $\mA,\mB$.  We will show that there exist absolute constants $c_1,c_2>0$ such that
\[
\Pr\Bigl\{\inf_{\mC\in\R^{k\times k}}
    \E_{\vx}\bigl\|\vx^{\top}\mA\mC-\vx^{\top}\mB\bigr\|_2^2
   < c_1\,m\,k\Bigr\}
\le 2\exp\bigl(-c_2\,m\,k\bigr).
\]

\medskip
\noindent See that, for any fixed $\mC\in\R^{k\times k}$,
\[
\E_{\vx}\bigl\|\vx^{\top}(\mA\mC-\mB)\bigr\|_2^2
=\tr\bigl((\mA\mC-\mB)^{\top}(\mA\mC-\mB)\E_{\vx}[\vx \vx^{\top}]\bigr) 
\]
\[
=\tr\bigl((\mA\mC-\mB)^{\top}(\mA\mC-\mB)\bigr)
=\|\mA\mC-\mB\|_F^2.
\]
Hence we have
\[
\inf_{\mC}\E_{\vx}\|\vx^{\top}\mA\mC-\vx^{\top}\mB\|_2^2
 = 
\inf_{\mC}\|\mA\mC-\mB\|_F^2.
\]

\medskip

\noindent Our goal now is to derive a lower bound for $\inf_{\mC}\|\mA\mC-\mB\|_F^2$. We will use a standard fact that random Gaussian matrices have full column rank almost surely (with probability $1$). Thus, the matrix $A$ has full column rank and the concatenated matrix $[A,B] \in \br^{m \times 2k}$ has full rank as well, implying that the columns of $A$ and $B$ are linearly independent of each other. 

The function $f(\mC) =\|\mA\mC-\mB\|_F^2$ has a unique minimizer $\mC^{*}= \mA^{\dagger}\mB$ given by the Moore-Penrose pseudoinverse $\mA^\dagger = (\mA^{\top}\mA)^{-1}\mA^{\top}$ which exists since $\mA$ has full column rank. The orthogonal projector onto its column space is given by
\[
\mP_{\mA} = \mA\,\mA^\dagger,
\quad \text{and let }
\mR = \mI_m - \mP_{\mA}.
\]

\noindent By expansion with $\mC=\mA^\dagger\mB$, see that
\[
\inf_{\mC}\|\mA\mC-\mB\|_F^2 
 =  \|\mA(\mA^{\top}\mA)^{-1}\mA^{\top}\mB-\mB\|_F^2
 = \|(\mI-\mP_{\mA})\,\mB\|_F^2
 = \sum_{j=1}^k\bigl\|\mR\,\vb_j\bigr\|_2^2,
\]
where $\vb_j\in\R^m$ is the $j$th column of $\mB$.

\medskip

\noindent Conditional on $\mA$, the matrices $\mP_{\mA}, \mR$ are fixed symmetric matrices of rank $k$ and $m-k$ respectively. Additionally, 
\[
\mP_{\mA}^2 = \mA(\mA^{\top}\mA)^{-1}\mA^{\top}\mA(\mA^{\top}\mA)^{-1}\mA^{\top} = \mA(\mA^{\top}\mA)^{-1}\mA^{\top} = \mP_{\mA}.
\]
This implies that all the eigenvalues of $\mP_{\mA}$ and consequently $\mR$ are either $0$ or $1$. Since $\mR$ has rank $m - k$,
\[
\|\mR\|_2\le1,
\qquad
\tr(\mR)=\|\mR\|_F^2=m-k.
\]

Each column $\vb_j\sim N(\vzero,\mI_m)$ is independent of $\mR$, and
\[
\E\|\mR\,\vb_j\|_2^2=\tr(R)=m-k. 
\]
 By the Hanson-Wright inequality
(Thm~\ref{thm:hansonwright}) , for any $t>0$,
\[
\Pr\Bigl\{\bigl|\|\mR\,\vb_j\|_2^2 - (m-k)\bigr|\ge t
       \Bigm| \mA\Bigr\}
 \le 
2\exp\Bigl(-c\,
  \min\bigl\{\tfrac{t^2}{\|\mR\|_F^2},\,\tfrac{t}{\|\mR\|_2}\bigr\}\Bigr)
\]
\[
\le 
2\exp\Bigl(-c\min\{t^2/(m-k),\,t\}\Bigr).
\]
Setting $t=(m-k)/2$ and using $m-k\ge m/2$, we get
\[
\Pr\bigl\{\|\mR\,\vb_j\|_2^2 < \tfrac12(m-k) \big| \mA\bigr\}
 \le 
2\exp\bigl(-c_3\,m\bigr).
\]
Because the $\vb_j$’s are independent, a union bound over $j=1,\dots,k$ yields
\[
\Pr\Bigl\{\exists\,j:\,\|\mR\,\vb_j\|_2^2<\tfrac12(m-k)\Bigr\}
 \le 
k\cdot2e^{-c_3m}
 \le 
2e^{-c_4\,m},
\]
where in the last inequality we used $k\le m/2$ and $n e^{-cn} = e^{-O(n)}$.

Thus, with probability $\ge 1-2e^{-cm}$, each
$\|\mR\,\vb_j\|_2^2\ge\tfrac12(m-k)$, so
\[
\inf_{\mC}\|\mA\mC-\mB\|_F^2
=\sum_{j=1}^k\|\mR\,\vb_j\|_2^2
 \ge 
k\cdot\tfrac12(m-k)
 \ge 
\frac{1}{2}mk.
\]

%% file: appendix/resources.tex
\section{Resources Used}

\paragraph{Compute hardware.}
All experiments were executed on NVIDIA \textbf{H100} GPUs.
\begin{itemize}
    \item \textbf{GPU model \& memory}\,: H100 (100\,GB HBM3).
    \item \textbf{GPU count per run}\,: 4.
    \item \textbf{Interconnect}\,: NVLink-2 / NVSwitch / PCIe 5.0.
    \item \textbf{CPU host}\,: dual-socket \emph{AMD EPYC 9334} (2 × 32 cores, 64 threads total).
    \item \textbf{System RAM}\,: \textbf{756 GB}.
    \item \textbf{Storage}\,: \textbf{1TB}.
\end{itemize}

\paragraph{Software stack.}
Ubuntu 22.04.5 LTS LTS, Linux kernel \textbf{5.15.0-135-generic}; \textbf{CUDA 12.6} (nvcc 12.6.68, build Wed Aug 14 2024); cuDNN \textbf{9.1.0}; PyTorch \citep{paszke2019pytorchimperativestylehighperformance} \textbf{2.4.0}; Python \textbf{3.12.3}.

\paragraph{Compute budget.}
Each full run consumed \(\approx\) 80 GPU-hours (4 × H100 × \emph{N} hours).  